\pdfoutput=1

\documentclass{article}

\usepackage[preprint]{acl}

\usepackage[utf8]{inputenc}
\usepackage{xcolor}

\usepackage{listings}

\usepackage{natbib}
\usepackage{verbatim}

\usepackage{amssymb}
\usepackage{amsthm}
\usepackage{amsmath}
\usepackage{amssymb}
\usepackage{amsthm}
\usepackage{makecell}
\usepackage{subcaption}
\usepackage{float}
\usepackage{booktabs}

\allowdisplaybreaks

\newcommand{\ypost}{y}%
\newcommand{\ypre}{Y}%
\newcommand{\head}{b}%
\newcommand{\headmac}[3]{b_{#1,#2}^{(#3)}}%
\newcommand{\seqlen}{n}%
\newcommand{\normFactor}{N}
\newcommand{\blowupInLemma}{N}%
\newcommand{\weightmac}[4]{\widehat{a}_{#1,#2}^{(#3,#4)}}%
\newcommand{\influence}{\operatorname{I}}%
\newcommand{\Blowup}{\operatorname{Blowup}}%

\newcommand{\flipBit}[1]{^{\oplus #1}}

\newcommand{\logMatPosPos}{A}%
\newcommand{\logMatPosTok}{C}%
\newcommand{\logMatTokPos}{B}%
\newcommand{\logMatTokTok}{W}%

\newcommand{\numLayers}{L}%

\usepackage[T1]{fontenc}

\usepackage{hyperref}

\usepackage{pslatex}
\usepackage[english]{babel}
\usepackage[utf8]{inputenc}
\usepackage{amsmath}
\usepackage{bm}
\usepackage{graphicx}
\usepackage{tikz}
\usepackage{xcolor}
\usepackage{url}
\usepackage{rotating}
\usepackage{amssymb}

\usepackage{amsthm}
 
\newcounter{theorem}

\newtheorem{corollary}[theorem]{Corollary}

\newtheorem{defin}[theorem]{Definition}
\newtheorem{remark}[theorem]{Remark}
\newtheorem{lemma}[theorem]{Lemma}
\newtheorem{thm}[theorem]{Theorem}
\newtheorem{fact}[theorem]{Fact}

\title{Why are Sensitive Functions Hard for Transformers?}
\author{Michael Hahn, Mark Rofin \\ Saarland Informatics Campus \\ Saarland University, Saarbr{\"u}cken, Germany \\ \texttt{\{mhahn, mrofin\}@lst.uni-saarland.de}}

\usepackage{natbib}
\usepackage{graphicx}

\begin{document}

\maketitle

\begin{abstract}
Empirical studies have identified a range of learnability biases and limitations of transformers, such as a persistent difficulty in learning to compute simple formal languages such as PARITY, and a bias towards low-degree functions. However, theoretical understanding remains limited, with existing expressiveness theory either overpredicting or underpredicting realistic learning abilities.
We prove that, under the transformer architecture, the loss landscape is constrained by the input-space sensitivity: Transformers whose output is sensitive to many parts of the input string inhabit isolated points in parameter space, leading to a low-sensitivity bias in generalization. 
We show theoretically and empirically that this theory unifies a broad array of empirical observations about the learning abilities and biases of transformers, such as their generalization bias towards low sensitivity and low degree, and difficulty in length generalization for PARITY.
This shows that understanding transformers' inductive biases requires studying not just their in-principle expressivity, but also their loss landscape.

\end{abstract}

\section{Introduction}

Given dramatic advances in machine learning applications powered by  transformer models, there has been substantial interest in understanding which functions are easier or harder to learn and represent using transformers.
Empirical research on both formal languages and synthetic functions has uncovered an intriguing array of learning biases, but theoretical understanding is lacking.
For instance,  \citet{abbe2023generalization} experimentally argued that heldout generalization is biased towards low-degree polynomials and \citet{bhattamishra2022simplicity} provided empirical evidence that transformers prefer to represent  functions of \emph{low sensitivity}, that is, functions that do not strongly depend on many input bits.
Perhaps the most prominent example of such learning biases is a consistent difficulty in learning the PARITY function, mapping bitstrings to their parity.
This function is extremely sensitive, in the sense that flipping any bit flips the string's parity.
Empirical studies have consistently found that training transformers to compute parities is difficult, and that solutions for shorter inputs do not generalize to longer inputs \citep[e.g.][]{bhattamishra2020ability, chiang2022overcoming, deletang2022neural, ruoss2023randomized}.
This stands in stark contrast to previously-popular reccurent models which easily fit PARITY with correct length generalization \citep{bhattamishra2020ability}.

\begin{figure}
    \includegraphics[scale=0.9]{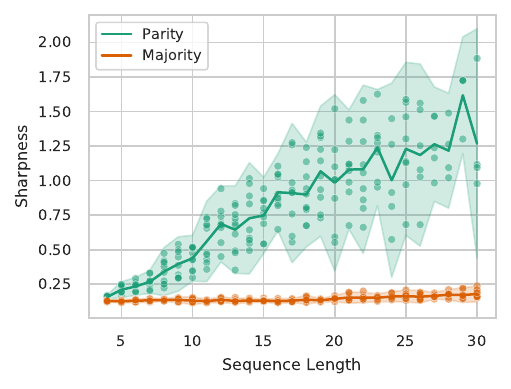}
    \caption{Training transformers on inputs of increasing length produces a steeper loss landscape for PARITY (as measured by average direction sharpness), while the loss landscape of MAJORITY does not show significant changes. 
    Our main result (Theorem~\ref{thm:lrho-bound}) provides a rigorous explanation for this phenomenon.}
    \label{exp:scaling-parity-majority}
\end{figure}

While a substantial amount of theoetical work has considered both the learnability \citep[e.g.][]{edelman2022inductive,ahn2023linear} and the expressiveness of transformers  \citep[e.g.][]{yun2019transformers,hahn2020theoretical,Yao_2021,hao2022formal,DBLP:journals/tacl/MerrillSS22,merrill2023logic, chiang2023tighter, strobl2023survey, strobl2023averagehard, angluin2023masked}, existing theoretical studies do not consistently explain such learning biases.
\citet{hahn2020theoretical} proved that, under two formal models of self-attention, no transformer can express PARITY at all input lengths. 
However, various other formal results showed that slightly relaxed assumptions about the transformer architecture resolved such expressiveness limitations.
Most notably, \citet{chiang2022overcoming} found that layer norm, by breaking the Lipschitz assumption used in \citet{hahn2020theoretical}'s Theorem 2, allows expressing PARITY in principle.
Simultaneously, they empirically confirmed that such a solution could not be practically found via (S)GD training.
Various other formal models of transformers \citep[e.g.][]{weiss2021thinking,merrill2023logic,merrill2023parallelism,strobl2023averagehard} can also express PARITY despite its empirical difficulty.
As already concluded by \citet{chiang2022overcoming}, these findings highlight a disconnect between expressive capacity and learnability: not all functions which transformers may express in principle are also learnt efficiently. Evidently, existing expressiveness theory for transformers is not able to consistently account for the practical learnability of problems under gradient descent.

Some prior work has studied the learnability of problems for transformers.
For example, \citet{edelman2022inductive}  bound the statistical capacity of the transformer architecture, showing that on those functions that transformers prefer to represent, they can generalize with good sample efficiency. 
Notably, they found that \emph{sparse} parities could indeed be learned well by transformers.
However, this result does not prove that PARITY, or other highly sensitive functions, are hard to learn, as that technique does not provide a direct characterization of which functions transformers prefer to represent. 
Other work has studied simplified setups such as  linear attention \citep[e.g.][]{ahn2023linear} or individual attention layers \citep[e.g.][]{sanford2023representational}. 

Here, we provide results that  have direct bearing on the learnability of PARITY and other sensitive functions, characterizing the loss landscape of transformers in terms of input-space sensitivity.
We formally prove that, for the transformer architecture, parameter settings achieving high sensitivity in input space are necessarily brittle, so that close neighbors in parameter space will usually define different (typically much less sensitive) functions when inputs are long.
As a consequence, transformers fitting high-sensitivity functions must inhabit very steep minima.
We argue that this explains both difficulty in training and length generalization for PARITY (observed by \citet{bhattamishra2020ability, deletang2022neural, ruoss2023randomized}), and a low-sensitivity and low-degree bias in random initialization and generalization (observed by \citet{abbe2023generalization, bhattamishra2022simplicity}).

\paragraph{Expressiveness theory does not account for learnability.}
While unique hard attention provably cannot represent PARITY \citep{hahn2020theoretical, hao2022formal, angluin2023masked}, more realistic upper bounds accounting for soft attention \citep{weiss2021thinking, merrill2023logic, merrill2023parallelism, strobl2023averagehard, chiang2022overcoming} leave the hardness of sensitive functions unexplained.
Not only does PARITY have transformers \citep{chiang2022overcoming}, but it can also be easily  represented in formalisms that have been suggested to meaningfully upper-bound the abilities of various formal models of soft-attention:\footnote{\citet{zhou2023algorithms}  suggest that PARITY may not be representable in the RASP-L model, though the expressiveness of RASP-L is not well understood.} 
\begin{fact}[Existing theory overpredicting abilities]\label{fact:overpredicting}
    Simple representations for PARITY, valid across all input lengths, exist in RASP \citep{weiss2021thinking}, uniform circuits with majority gates \citep{merrill2023parallelism, strobl2023averagehard}, and FO[M] \citep{merrill2023logic}. 
\end{fact}
We prove this in Appendix~\ref{sec:appendix:construct-parity}.
Thus, existing expressiveness bounds do not account for the difficulty that transformers encounter in learning sensitive functions, in particular given that previously-popular recurrent models do not encounter this difficulty. 
Another family of results consists of Lipschitzness bounds \citep{hahn2020theoretical, li2023stability}, which bound the influence that any individual input bit has on the output of a transformer.
These turn out to \emph{underpredict} the abilities of transformers:

\begin{fact}[Existing theory underpredicting abilities]\label{fact:underpredicting}
By results of \citet{hahn2020theoretical, li2023stability}, the following holds:
Consider a transformer without layer norm.
If $x, x' \in \{\pm 1\}^n$ differ only in the $i$-th bit, then at any other position $j\neq i$, the output of a transformer differs only by $\mathcal{O}(\frac{1}{\seqlen})$.
\end{fact}
This accounts for the difficulty of learning PARITY.
But the bound suggests even simple sparse functions, such as FIRST (the language $1(0|1)^*$) to be difficult, but transformers learn these well \citep{bhattamishra2022simplicity, edelman2022inductive}. Indeed, \citet{chiang2022overcoming} note that the bound is overcome by layer norm or input-length-dependent scaling of attention logits, which enable modeling of sparse functions.

We will show that the observed low-sensitivity bias can be understood in terms of the \emph{loss landscape}: while transformers can express highly sensitive functions, such transformers are isolated in parameter space, and minima interpolating a sensitive function are very sharp.
Indeed, we prove that tiny perturbations of a highly sensitive transformer tend to define, when inputs are sufficiently long, very different functions with much lower sensitivity.

\section{Model of Transformers}

We will focus on boolean functions.
Following the conventions in the Analysis of Boolean Functions literature \citep{odonnell2014analysis} in modeling bitstrings as elements of $\{-1,1\}^\seqlen$, we assume the alphabet $\Sigma = \{-1,1\}$, with word embeddings $e(-1), e(+1) \in \mathbb{R}^d$.
There further are positional encodings $p_1, p_2, p_3, \dots \in \mathbb{R}^d$.
At the zero-th layer, token and positional encodings are added: $y_i^{(0)} := e(x_i) + p_i$ ($i=1, \dots, n$), where $x \in \{\pm 1\}^\seqlen$ is the input string.

A transformer has a fixed number $L$ of \textbf{layers}; the \textbf{activations} $y_i^{(k)} \in \mathbb{R}^d$ at position $i$ of the $k$-th layer ($k=1, \dots, \numLayers$) are defined as follows.
Each layer has a set of $H$ \textbf{attention heads}; we first compute attention scores for the $h$-th head:
\begin{align*}
    a_{i,j}^{(k,h)} =& (K_{k,h} y_j^{(k-1)})^T Q_{k,h} y_i^{(k-1)} \\
    \weightmac{i}{j}{k}{h} =& \frac{\exp(a_{i,j}^{(k,h)})}{\sum_s a_{i,s}^{(k,h)}}
\end{align*}
where $K_{k,h}$ (``key''), $Q_{k,h}$ (``query'') are $\in \mathbb{R}^{d\times d}$.
The  activation of the head is computed by weighting according to attention weights $\hat{a}_{i,j}^{(k,h)}$, and applying a linear transformation $V$ (``value''):
\begin{equation}
    \headmac{i}{h}{k} =  \sum_{j=1}^{\seqlen}\hat{a}_{i,j}^{(k,h)}  V_{k,h} \ypost_j^{(k-1)} 
\end{equation}
The  per-position activations are then computed as
\begin{equation}\label{eq:def:ypre}
\ypre_i^{(k)} :=  f^{MLP}\left(\ypost_i^{(k-1)} +  \sum_{h=1}^H \headmac{i}{h}{k}\right)
\end{equation}
where $f^{MLP}$ is a one-layer MLP with a skip-connection.
Transformers additionally implement layer norm \citep{DBLP:journals/corr/BaKH16}:
\begin{equation}\label{eq:def:layer-norm}
    LayerNorm(y) := \frac{y-mean(y)}{\sqrt{\sigma^2(y)+\epsilon}}
\end{equation}
where $\epsilon\geq 0$ is a hyperparameter ensuring numerical stability, and $\sigma^2(\cdot)$ denotes the variance.
By design, $\|LayerNorm(y)\|_2 \leq \sqrt{d}$, with equality at $\epsilon=0$.
Transformer variants differ in where exactly layer norm is applied \citep[e.g.][]{takase2022layer}; we here assume for notational simplicity that layer norm applies after the MLP, but the details are irrelevant to our results, provided layer norm applies at least once. We thus set:
\begin{equation}\label{eq:layer-norm-and-mlp}
    \ypost_i^{(k)} := LayerNorm\left(\ypre_i^{(k)}\right) 
\end{equation}
Of key importance will be the the normalization factor:
\begin{equation}\label{eq:def:ln-fctor}
    \normFactor_i^{(k)} := \frac{1}{\sqrt{\sigma^2(\ypre_i^{(k)})+\epsilon}} 
\end{equation}
Our theoretical results will link $\normFactor_i^{(k)}$ both to input-space sensitivity and parameter-space sharpness:
We will find that large values of $\normFactor_i^{(k)}$ can increase expressive capacity, but at the price of increased brittleness.

Finally, we assume that predictions are made by  $T := v_{out}^T \cdot \ypost_\seqlen^{(\numLayers)}$ for some parameter $v_{out} \in \mathbb{R}^d$. 
Throughout, we will add the input string $x \in \{\pm 1\}^\seqlen$ as an argument when needed for disambiguation, e.g., writing $T(x)$ for the overall prediction made on $x$.

\section{Average Sensitivity}

Our results are centered around \emph{average sensitivity}, a simple but foundational complexity metric for functions on the Boolean cube~\citep[e.g.][]{kahn1988the,de2008brief,odonnell2014analysis}:
\begin{defin}
For a bitstring $x \in \{\pm 1\}^n$ and a function $f : \{\pm 1\}^{\seqlen}\rightarrow \mathbb{R}$, we write
\begin{equation}\label{eq:def:s}
s(x,f) := \frac{1}{4} \sum_{i=1}^n |f(x) - f(x^{\flipBit{i}})|^2
\end{equation}
where $x^{\flipBit{i}}$ denotes the bitstring $x$ with the $i$-th bit flipped.
The average sensitivity for inputs of length $\seqlen$ is
\begin{equation}
    as_n(f) := \frac{1}{2^n} \sum_{x \in \{\pm 1\}^\seqlen} s(x,f)
\end{equation}
\end{defin}
If $f$ maps to $\{\pm 1\}$, then $s(x,f)$ is the number of Hamming neighbors of $x$ on which $f$ flips.
This definition of $as_n(f)$ corresponds to the ``total influence'' from \citet[][Def. 2.27]{odonnell2014analysis}.
We explicitly define average sensitivity relative to input length $\seqlen$, as we will investigate the behavior of transformers performing a single function $f$ across varying input lengths.
The use of squared distances, rather than simple absolute distances, ensures that results about $as_\seqlen(f)$ transfer to results about degree profiles \citep{abbe2023generalization}, which we will later investigate (Eq.~\ref{eq:sensitivity-fourier}).

Average sensitivity is a general complexity metric with wide-ranging applications in theoretical computer science \citep[e.g.][]{jukna2012boolean}.
It is closely linked to the Fourier analysis on the Boolean cube \citep{de2008brief}, and is an average-case version of a family of sensitivity measures, closely related to other natural metrics such as decision tree depth and polynomial degree \citep{hatami2010variations}.
Both average sensitivity itself \citep{bhattamishra2022simplicity} and the Fourier structure \citep{abbe2023generalization} have been empirically linked to transformers' generalization behavior. We will ground these empirical findings by relating average sensitivity to loss landscapes for the transformer architecture.

\paragraph{Example Functions}
Our theoretical results will apply to general functions on the Boolean cube.
In order to ground these, we will illustrate transformers' low-sensitivity bias at the example of a few natural functions which have played a role in the theoretical literature of transformers or are otherwise illustrative of variability in average sensitivity.

\textbf{PARITY} indicates whether the number of ones in a bitstring is even (output 1) or odd (output -1); over the input space $\{\pm 1\}^\seqlen$ and output space $\{\pm 1\}$, it can be formally defined as the map $x \mapsto \prod_{i=1}^\seqlen x_i$.
As flipping any input bit flips the output, $as_\seqlen(f) = n$.
As described above, this function can in principle be represented by transformers, but has empirically been found to be very hard to learn. 

\textbf{MAJORITY} maps $x \in \{\pm 1\}^\seqlen$ to 1 if $\#\{i: x_i=1\} > \seqlen/2$, and $-1$ else.
Transformers show good length generalization \citep{DBLP:journals/tacl/MerrillSS22, zhou2023algorithms}.
It is known that $as_\seqlen(f)  = \Theta(\sqrt{\seqlen})$ \citep[Ex. 2.22]{odonnell2014analysis}.
However, $s(f,x) = \seqlen$ whenever the ones and zeros in $x$ are almost fully balanced.

\textbf{FIRST} maps $x$ to its first bit, $x_1$.
As only the first bit matters, $as_\seqlen(f) = 1$. 
It is a simple example of a \emph{sparse} function; more generally, a $k$-PARITY is a restriction of the PARITY function to only $k$ inputs, where $k$ is a constant.
Transformers learn such sparse functions well \citep{bhattamishra2022simplicity, edelman2022inductive}. 

\textbf{MEAN} maps $x \mapsto \frac{1}{\seqlen} \sum_{i=1}^\seqlen x_i \in [-1,1]$. We have $as_\seqlen(f) = \frac{1}{\seqlen}$.

The PARITY function, when varying the number of bits considered, is a universal basis for Boolean functions, in the sense that any function $\{\pm 1\}^n \rightarrow \mathbb{R}$ can be represented as a linear combination of parities applying to different subsets of $\{x_1, \dots, x_\seqlen\}$.
Functions are more sensitive when parities applying to larger subsets appear in this representation.
We will investigate this connection further below.

\section{Lower Bounds for Sensitive Functions}

We first prove that representing sensitive functions with transformers requires large parameter norms and, when inputs get longer and longer, highly unbounded normalization factors $\normFactor_i^{(k)}$ in layer norm (\ref{eq:def:ln-fctor}).
We start from the global Lipschitzness bounds developed by \citet{hahn2020theoretical, edelman2022inductive, li2023stability}, but  obtain more fine-grained average-case and high-probability bounds. These will then form the basis of our characterization of loss landscape sharpness around sensitive transformers.
Our bounds will include a constant $C$ that is the product of
\begin{equation}\label{eq:def:c}
    \exp\left(4d \max_h \sum_{i=2}^L \|K_{i,h}^TQ_{i,h}\|_{2}\right)
\end{equation}
and a term polynomial in $H$, $d$, the spectral norms of all parameter matrices appearing in the transformer, and the maximum norm of any positional or word embedding.
See (\ref{eq:def:c-appendix}) in the Appendix for formal definition.

Existing Lipschitzness bounds (Fact~\ref{fact:underpredicting}) imply a bound along the lines of\footnote{Lipschitzness bounds by \citet{edelman2022inductive} are not directly applicable here, as these authors consider Lipschitzness in the \emph{parameter space}, not the effect of changes in the \emph{input space} as Theorem 3. See also Appendix, Remark~\ref{remark:comparison-head-edelman}.
}
\begin{equation}\label{eq:theorem-first}
s(f,x) \leq  \frac{C \exp(4d\max_h\|K_{1,h}^TQ_{1,h}\|_{2})}{\epsilon^{\numLayers/2}}
\end{equation}
uniformly for each $x \in \{\pm 1\}^\seqlen$.
\citet{li2023stability} noted that the exponential dependency on the spectral norm of the key-query matrix is unavoidable for a global Lipschitzness bound.
Our first result is that this dependency can be eliminated at the input layer for the vast majority of inputs, at the cost of a logarithmic factor, leading to a bound of the form (assuming $\epsilon > 0$ in (\ref{eq:def:layer-norm})) 
\begin{equation}\label{eq:theorem-second}
	s(f,x) \leq C \frac{\log{\seqlen}}{\epsilon^{\numLayers/2}}
\end{equation}
for $1-H\seqlen^{-2}$ of inputs.
We show this using a concentration bound argument, combining a Chernoff bound applying to each attention weight individually, with a union bound over all attention weights (Appendix, Lemma~\ref{lemma:attention-bound-first-layer}).
Next, we address the role of layer norm.
\citet{chiang2022overcoming} showed that, at $\epsilon\rightarrow 0$ in (\ref{eq:def:layer-norm}), layer norm enables transformers to represent PARITY. Lipschitzness bounds in terms of $\epsilon$ (\ref{eq:theorem-second}) cease being meaningful in this limit. 
We thus study the layer-norm induced blowup in more detail. 
In each layer, we consider the maximum blowup, given an input string $x \in \{\pm 1\}^\seqlen$:
\begin{equation}\label{eq:def:tau}
    \tau^{(k)}(x) := \max_{w=1,\dots,\seqlen} \{ 1+\normFactor_w^{(k)}(x)\} 
\end{equation}
The addition of 1 is for technical reasons (Appendix, Lemma~\ref{lemma:influence-layer-norm}); it has little impact in the cases relevant to our results, which will be when $\tau^{(k)}(x) = \omega(1)$.
We then write $\Blowup(x) := \prod_{k=1}^L \tau^{(k)}(x)$, an upper bound on the product of the successive layer-norm-induced blowups when going from the input to the output.
When $\epsilon=0$ in (\ref{eq:def:layer-norm}), $\Blowup(x)$ can be arbitrarily large.
Foreshadowing Theorem~\ref{thm:lrho-bound}, we will find that large values of $\Blowup(x)$ create very sharp minima.

Our first theorem localizes the layer norm blowup to the Hamming neighborhoods of sensitive inputs:
\begin{thm}[Local Bounds on Layer Norm Blowup]\label{thm:sensitivity-blowup-bound}
Consider a transformer with layer norm at arbitrary $\epsilon \geq 0$. 
With probability $1-\frac{H}{\seqlen^{2}}$ over the choice of $x \in \{\pm 1\}^\seqlen$, we have
\begin{equation}\label{eq:main-formula-localized}
    \frac{s(f,x)}{C\sqrt{\seqlen\log\seqlen}}  
    \leq  \Blowup(x)^2 +\frac{1}{\seqlen} \sum_{i=1}^{\seqlen} \Blowup(x\flipBit{i})^2 
\end{equation}
\end{thm}
The proof is in Appendix~\ref{sec:app:pointwise}.
This permits us to state a bound on average sensitivity, in terms of the average layer norm blowup:
\begin{corollary}[First Main Result]\label{thm:bigtheorem}
Consider a transformer with layer norm at arbitrary $\epsilon \geq 0$.
Then
\begin{equation}\label{eq:theorem-third}
 C \cdot  \mathbb{E}[\Blowup(x)^2] \geq       \frac{as_\seqlen(f)}{\sqrt{\seqlen\log{\seqlen}}}  -   \frac{H}{n}
\end{equation}
where the expectation is over the uniform distribution over $x\in\{\pm 1\}^\seqlen$.
\end{corollary}
The proof is in Appendix~\ref{sec:app:average}.
Note that $\frac{H}{n}$ is small when $n$ is large, and the bound is dominated by $\frac{as_\seqlen(f)}{\sqrt{\seqlen\log{\seqlen}}}$.
This result thus shows a tradeoff between parameters and LN blowup: at least one of them needs to be large to represent a sensitive function. When the sensitivity depends on the input length and grows faster than $\sqrt{n \log n}$, this tradeoff changes with the input length, requiring larger parameters or larger layer norm blowup as the input length increases.

Let us investigate the implications for the functions introduced in Section above.
For PARITY, $as_\seqlen(f) = n$, and
(\ref{eq:theorem-third}) predicts
\begin{equation}
C \cdot  \mathbb{E}[\Blowup(x)^2] =  \Omega\left(\frac{\sqrt{\seqlen}}{\log{\seqlen}}\right) = \omega(1)
\end{equation}
showing that, for fixed parameters, the layer norm blowup needs to increase as the input length increases.
For the other functions, the bound is $O(1)$:
For FIRST, $s(f,x) = as_\seqlen(f) = 1$, and  the RHS of (\ref{eq:theorem-third}) is $O(1)$. 
Indeed, sparse functions can be modeled well by a family of transformers where the logits in the input layer scale with $\log{\seqlen}$ \citep{edelman2022inductive, chiang2022overcoming}.
Unlike the prior Lipschitzness bounds (\ref{eq:theorem-first}), these scaled logits do not contribute to $C$ -- our new bound is thus consistent with the ease with which transformers learn sparse functions.
For MEAN, $s(f,x) \sim \frac{1}{\seqlen^2}$; again no blowup is predicted.
For MAJORITY, as $as_\seqlen(f) \sim \sqrt{\seqlen}$, no nontrivial average blowup is predicted.
However, $s(f,x) = \seqlen$ whenever the ones and zeros in $x$ are almost fully balanced; on such strings or their Hamming neighbors, (\ref{eq:main-formula-localized}) predicts $Blowup = \Omega(\seqlen^{1/4})$.

\section{Sensitive Transformers are Brittle}

Leveraging Corollary~\ref{thm:bigtheorem}, we now show that transformers expressing sensitive functions must be very sensitive to small perturbations of model parameters.
That is, \emph{sensitivity in input space} entails \emph{sensitivity in parameter space}.
An important consequence is that, for any highly sensitive function, any interpolating minima will be very sharp.
This property is nontrivial, as seen from the fact that it disappears when adding a scratchpad (see below).

Given a parameter vector $\theta$ defining the transformer $T_\theta$, the \emph{average direction sharpness} is \citep[e.g.][]{DBLP:journals/corr/abs-2211-05729,jiang2019fantastic,foret2020sharpness}
\begin{equation}\label{eq:def:lrho}
    L_{\rho,\seqlen}(T) := \mathbb{E}_{x \in \{\pm 1\}^\seqlen} \mathbb{E}_{\|\Delta\|_2 = \rho} (T_{\theta+\Delta}(x) - T_{\theta}(x))^2
\end{equation}
where the second expectation is over the radius-$\rho$ sphere in the space of parameter vectors.
$L_{\rho,\seqlen}$ quantifies the change in $T_\theta$ when perturbing $\theta$ by a size-$\rho$ vector in a random direction.
Lower bounds on $L_{\rho,\seqlen}$ (Theorem~\ref{thm:lrho-bound}) immediately entail lower bounds on other common sharpness measures, such as worst-case sharpness, and (in the limit $\rho\rightarrow 0$) the trace of the loss function's Hessian at $\theta$.

In defining the parameter vector $\theta$ in (\ref{eq:def:lrho}), we exclude positional encodings, both because they are often frozen, and because their number depends on $\seqlen$, hindering fair comparison across $\seqlen$.

Our next result lower-bounds $L_{\rho,\seqlen}$ in terms of the average sensitivity, provided the transformer is sufficiently wide in relation to its depth:
\begin{thm}[Second Main Result]\label{thm:lrho-bound}
Let $T_\theta$ be a transformer where $d > 12L$.
Assume $T_\theta(x) \in [-1,1]$ for each $x$.
Then:
\begin{equation}\label{eq:lrho-blowup}
  \lim_{\rho \rightarrow 0} \liminf_{n\rightarrow \infty}   L_{\rho,\seqlen}(T) \geq \liminf_{n\rightarrow \infty} \frac{as_n(T_\theta)}{2n} - \numLayers\exp(-\Omega(d))
  \end{equation}
  Here, ``$\Omega(d)$'' scales positively with $d$.
  If $T_\theta$ has Boolean outputs ($T_\theta(x) \in \{\pm 1\}$), 
  then the factor ``2'' can be eliminated from (\ref{eq:lrho-blowup}).
\end{thm}
Informally, this theorem says that, when $as_n(T_\theta) \sim \seqlen$, then even tiny perturbations to the parameters will, in expectation, lead to a substantial change in predictions on long inputs.
This means that, as the input gets longer, the Hessian of the mean-squared loss at the minimizer fitting a sensitive function has unboundedly large entries.

See Appendix~\ref{appendix:sec:theorem-sharpness} for the proof.
The key idea of the proof is that, if $T_\theta$ is very sensitive in input space, small perturbations to the parameters usually lead to a large drop in sensitivity when $\seqlen$ is large, because they lead to large changes in the layer-norm induced blowup that is needed to represent high-sensitivity functions.
As a consequence, transformers computing sensitive functions are isolated in parameter space. 

The theorem applies when $d$ is substantially larger than $\numLayers$, as is indeed true of typical transformers (e.g., LLaMa 7B has $d=4096$ and $\numLayers=32$).
The convergence of the limits is slower when the parameter-norms, as summarized by $C$, are larger, as larger parameters can increase sensitivity.
However, remarkably, for any fixed transformer, $C$ becomes irrelevant for $L_{\rho,\seqlen}$ in the limit where $n\rightarrow \infty$ (\ref{eq:lrho-blowup}).

While we stated Theorem~\ref{thm:lrho-bound} for an individual transformer, the same statement holds for families of transformers where weights may depend on $\seqlen$, as long as $C$ remains bounded.
A consequence is that scaling attention logits with $\log n$ (used to represent sparse functions by \citet{edelman2022inductive, chiang2022overcoming}) will, at least in the input layer, not mitigate the difficulty of sensitive functions.

\section{Implications}
We discuss how Theorem~\ref{thm:lrho-bound} unifies a range of diverse empirical findings about the behavior of transformers.
\paragraph{Difficulty of PARITY}
For PARITY, $L_{\rho,\seqlen}$ converges to $1$ for large $d$, showing that arbitrarily small perturbations to a transformer computing PARITY will lead to a high loss for sufficiently long inputs.
Previously, \citet{chiang2022overcoming} noted that, for their hand-constructed transformer representing PARITY, small changes to a specific parameter led to a large increase in loss, suggesting that this made the solution impossible to reach with SGD.
Theorem~\ref{thm:lrho-bound} shows that this phenomenon is unavoidable for any transformer representing a high-sensitivity function.
For functions with sublinear average sensitivity, Theorem~\ref{thm:lrho-bound} entails no nontrivial lower bound on sharpness, and no such phenomenon is predicted.

\paragraph{Random Initialization}
A key step in Theorem~\ref{thm:lrho-bound} is to show that $s(T_{\theta+\Delta},x)$ is bounded with very high probability over the choice of $\Delta$ (Appendix, Eq. (\ref{eq:raw-sens-perturbed-over-n})); this immediately entails that high-sensitivity transformers can only inhabit a small volume in parameter space.
This explains why randomly initialized transformers empirically show low average sensitivity, more so than recurrent networks \citep{bhattamishra2022simplicity}. 

\paragraph{Length Generalization}
An important corollary is that, for a sensitive function, length generalization requires exact match to the minimum: the slightest deviation from the exact minimum will, in expectation, lead to failure when inputs get sufficient long, even if the minimum itself represents a length-generalizing solution.
This provides, for the first time, a rigorous explanation why, despite the in-principle existence of length-generalizing transformers, transformers struggle with length generalization for PARITY \citep[e.g.][]{bhattamishra2020ability}.

\paragraph{Generalization on Unseen Data}
Another corollary is that, in expectation, the training loss landscape around an interpolating minimum places a constraint on a function's overall sensitivity, \emph{even if not a single pair of Hamming neighbors} are in the training set.
This is because (\ref{eq:lrho-blowup}) remains true, on average across randomly selected training sets, if replacing the expectation over the full input space with an expectation over the training set in (\ref{eq:def:lrho}).
This means that, for long inputs, flat minima of the training loss generalize with bounded sensitivity.
To the extent that gradient-based training tends to find flatter minima \citep[e.g.][]{Damian2023SelfStab}, this  provides a theoretical justification for the empirical result that transformers'  generalization behavior, when trained on a Boolean function on some training subset of $\{\pm 1\}^\seqlen$, shows a strong bias towards low average sensitivity \citep{bhattamishra2022simplicity}.

\citet{abbe2023generalization}  proposed a \emph{min-degree} bias, that is, a generalization bias towards functions that are linear combinations of functions that each depend on only a few inputs. Any function $f : \{\pm 1\}^{\seqlen}\rightarrow \mathbb{R}$ can be uniquely written as a linear combination of the multilinear monomials $\chi_P(x) := \prod_{i \in P} x_i$, where $P \subseteq \{1, \dots, n\}$: $f(x) = \sum_P \lambda_P \chi_P(x)$.
The coefficients $\lambda_P$ define the Fourier-Walsh transform of $f$.
For $\chi_P$, both its degree as a polynomial, and its average sensitivity, are $|P|$. 
The degree profile, as defined by \citet{abbe2023generalization}, is the tuple $(d_1, ..., d_n)$ where $d_i = \sum_{P : |P|=i} |\lambda_P|^2$. Minimization of degree profile then refers to setting as many of the later entries to zero, and minimizing the size of the last nonzero entry.
\citet{abbe2023generalization} proved inductive biases towards functions with a low degree profile for a random features model and diagonal networks. Evidence in the case of transformers was limited to empirical data on three example functions.
Our results entail prove an bound on degree profiles for the full transformer architecture, because
average sensitivity is a summary statistic of the degree profile \citep{odonnell2014analysis}:
\begin{equation}\label{eq:sensitivity-fourier}
    as_n(f) = \sum_{P \subseteq \{1,\dots,n\}} \lambda_P^2 |P| = \sum_{i=0}^{\seqlen}i \cdot d_i
\end{equation}
Hence, functions with degree profile assigning substantial weight to degrees of order $\sim \seqlen$ are brittle in parameter space, corresponding to very sharp minima.

\paragraph{Intermediate Steps Reduce Sensitivity}
%
PARITY can be solved well with a scratchpad \citep{anil2022exploring,liu2022transformers}.
Existing theoretical accounts of the benefit of intermediate steps for transformers' expressive capacity \citep[e.g.][]{merrill2023expresssive, feng2023towards} do not account for the benefit of intermediate steps for PARITY-like problems: While the theoretical models of transformer expressivenes used in these studies do predict versions with intermediate steps to be easy, they do not predict that computing PARITY in a single step would be hard, due to Fact 1. 
The concept of average sensitivity provides a simple explanation for the benefit of intermediate steps.
Formally, we can consider the problem of simulating a finite automaton with state set $\mathcal{X}$ either translating to the final state ${t}_n$ in one go (standard), or to autoregressively translate it into a sequence of states ${t}_1, ..., {t}_n$ (scratchpad). Then (proof in Appendix~\ref{sec:proof:scratchpad}):

\begin{thm}
Simulating an automaton with scratchpad has sensitivity $\mathcal{O}(1)$ for each autoregressive step.
\label{thm:scratchpad}
\end{thm}

\section{Experiments}

\subsection{Setup}

We conducted experiments to test the predictions made by our theory, specifically assessing predictions regarding loss landscapes and sharpness of the minima.
In all experiments, we use the transformer encoder architecture, using the default implementation in PyTorch \citep{pytorch}. Each model is trained to fit a function $f$ on a specific sequence length $\seqlen$. Each training input $x$ is generated uniformly from $\{\pm 1\}^\seqlen$, and each input bit, treated as a separate token, is embedded using learned token and positional encodings.\footnote{Code is available at \url{https://github.com/lacoco-lab/sensitivity-hardness}.}

The representation of the last token is passed to a linear layer to generate the prediction $T_\theta(x)$. Parameters $\theta$ of the transformer are optimized for MSE loss between $T_\theta(x)$ and $f(x)$ using AdamW \citep{Loshchilov2017DecoupledWD}. For full details on hyperparameters and training setup, refer to Appendix \ref{app:hyperparameters}.

In implementation, we assumed versions of the functions outputting to $\{0,1\}$; we rescaled sharpness values accordingly for comparability with the theory. 

We analyzed models using  the following metrics:
\begin{enumerate}
    \item \textbf{Parameter Norm.} We compute the L2 norm of the entire model's parameter vector,  excluding positional encoding matrices. 
    We discard the norm of positional encodings, so that the norms of the models trained for different sequence lengths are comparable.
    \item \textbf{LayerNorm Blowup.} This metric is computed by computing the maximum normalization factor (\ref{eq:def:ln-fctor}) across the entire layer in each application of layer norm, and take the product over all applications of layer norm. This essentially corresponds to $\Blowup$.\footnote{In the theory, Blowup has an additional 1+... in each factor for technical reasons. This difference is immaterial, as we are interested in situations where $Blowup = \omega(1)$.}
    \item \textbf{Sharpness.} 
    In order to avoid committing to any specific $\rho$, we sample $\Delta$ in (\ref{eq:def:lrho}) not from the radius-$\rho$-sphere, but from a mean-zero Gaussian with STD $\rho = 0.02$.
    We estimate using  $N_{s,p}$ perturbations and $N_{s, b}$ input strings $x$. This provides results equivalent to $L_{\rho,\seqlen}$ in the $\rho\rightarrow 0$ asymptotic, while avoiding committing to any specific $\rho$.
\end{enumerate}

As we are interested in properties of models that compute given functions, runs that did not converge (evaluation MSE higher than $10^{-3}$) were discarded.

\begin{figure}
    \centering
    \includegraphics[scale=0.9]{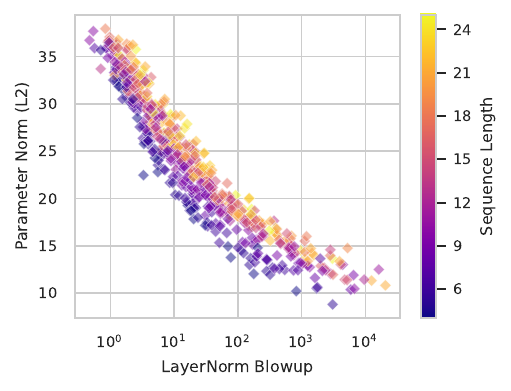}
    \caption{The tradeoff between parameter norm of Transformers trained to approximate PARITY and the blowup of their Layer Normalization layers. The tradeoff depends on the input length; blowup or parameter weights need to increase with the input length (in accordance with Corollary \ref{thm:bigtheorem}). This length dependency is not observed with low sensitivity functions (Appendix, Figures \ref{exp:tradeoff-all-functions-comparison} and \ref{exp:tradeoff-4-lengths}).}
    \label{exp:tradeoff-parity}
\end{figure}

\subsection{Results}

\paragraph{Higher Sensitivity Implies Sharper Minima.}

In this experiment, we train transformers to fit $f \in \{\text{PARITY}, \text{MAJORITY}, \text{FIRST}, \text{MEAN}\}$ on sequence lengths from 4 to 30. For each function and sequence length, we retrain the model 10 times from different random initializations.

For PARITY, sharpness stably increases with the input length (Figure~\ref{exp:scaling-parity-majority}). 
For the other functions, whose sensitivity grows more slowly with $\seqlen$, (a) the absolute value of sharpness is orders of magnitude lower than for PARITY; (b) there is little increase with $\seqlen$.
More results are shown in Appendix \ref{appendix:exp_results}.

\begin{figure}
    \centering
    \includegraphics[scale=0.9]{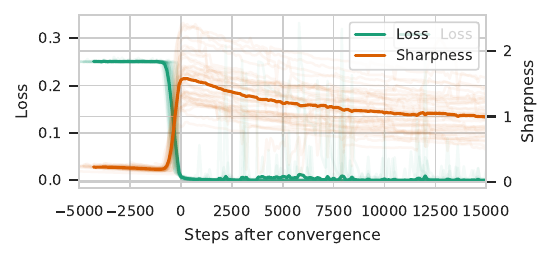}
    \caption{During training a transformer on PARITY, a sudden drop in the loss coincides with an increase in sharpness. Sharpness decreases again in further training, but asymptotes to a nontrivial value (Appendix, Figure~\ref{fig:dynamics-100k}).  See corresponding curves for weight norm and LN Blowup in Appendix, Figure \ref{exp:dynamic-main}.}
    \label{exp:dynamic-small}
\end{figure}

\paragraph{Tradeoff between Weight Norm and Layer Norm Blowup.}

At any fixed input length $\seqlen$, high sensitivity can be achieved by a combination of large weights and a large LayerNorm blowup. 
By Theorem \ref{thm:bigtheorem}, the product of $C$ and squared blowup is bounded from below with some value $B_n(f)$. Hence, the product of $\sqrt{C}$ and blowup is bounded with $\sqrt{B_n(f)}$, and the sum of $\frac{1}{2} \log C$ and $\log \mathrm{Blowup}$ is bounded with $\frac{1}{2} \log {B_n(f)}$. 
However, $C$ depends exponentially on parameters, and thus we expect the parameter norm to trade off with the logarithm of the layer norm blowup.
Moreover, for PARITY the value of $B_n(f)$ increases with $n$, and therefore sum of parameter norms and the logarithm of the blowup should also increase with $n$.

To test this prediction, for each function $f$ we train a set of models with varying sequence lengths $\seqlen$,   weight decay and learning rate parameters. It allows us to obtain datapoints with diverse values of LN Blowup and parameter norm.

Results for PARITY can be seen in Figure \ref{exp:tradeoff-parity}, and for other functions in Figure \ref{exp:tradeoff-all-functions-comparison}. For all functions, there is a clear log Blowup-Parameter Norm tradeoff.
For PARITY, the shape of the tradeoff indeed depends on $\seqlen$, with transformers trained for high $n$ located above others in the log Blowup-Parameter Norm coordinates. For other functions, dependency on $\seqlen$ is not visible, at least at this range of $\seqlen$.

\paragraph{The Limits of Transformers' Generalization.}
As discussed above, Theorem~\ref{thm:lrho-bound} predicts that transformers will generalize with low sensitivity, as low-sensitivity functions will tend to have flatter minima.

To test this prediction, we created random functions $f := \{\pm 1\}^\seqlen \rightarrow \{\pm 1\}$, and sampled random training sets from $\{\pm 1\}^\seqlen$ of size 128/256/512. We fixed $\seqlen=10$.
For each $f$, we train a transformer $T_1$ on the training set and use it to label the whole input space of sequences of length $n$ (1024 objects in our case).
Replicating \citet{bhattamishra2022simplicity}, these extrapolated functions $T_1$ have lower average sensitivity than the original functions $f$, indicating a low-sensitivity bias in generalization (Figure~\ref{fig:generalization}).
Now, in order to directly compare the sharpness of minima corresponding to the true function $f$ and the extrapolated function $T_1$, we trained new transformers to fit both functions on the entire dataset $\{\pm 1\}^\seqlen$, and measured the sharpness for these two new transformers. The results were averaged over 10 random functions $f$, 10 training sets per $f$ and training set size, and 5 new transformers per each $T_1$.
Sharpness was indeed lower for the transformer fitting the extrapolated functions than for the  transformer matching the original random functions $f$.
This also held when measuring sharpness only on the training set (Appendix, Figure~\ref{fig:generalization-sharpness}).


\begin{figure}
\includegraphics[scale=0.9]{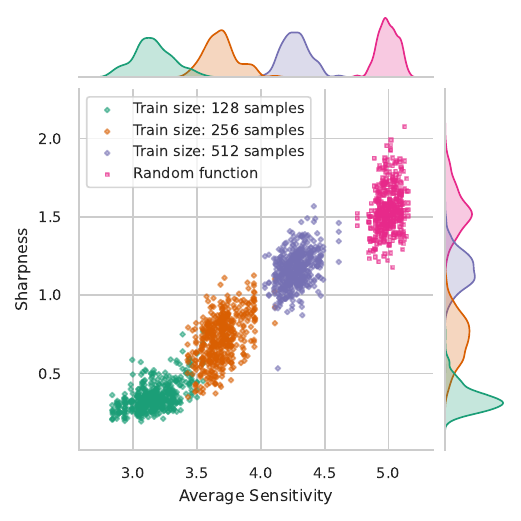}
    \caption{Generalization: When trained on data from a random Boolean function a subset of $\{\pm 1\}^n$ (here: n=10), transformers generalize with reduced sensitivity compared to the actual function. The solutions found have lower sharpness than a solution fitting the actual function. When the training size is smaller, the inferred function is less constrained, and learnt functions have even lower sensitivity.}\label{fig:generalization}
\end{figure}

\paragraph{Scratchpad Eliminates Sharpness.}
By Theorem \ref{thm:scratchpad}, sensitivity of each autoregressive step when computing PARITY with scratchpad is $O(1)$. Hence, Theorem $\ref{thm:lrho-bound}$ provides no nontrivial lower bound for $L_{\rho,\seqlen}(T)$. 
We trained an Encoder-Decoder Transformer, predicting PARITY of $i$-th substring on $i$-th autoregressive step: ${t}_i = \mathrm{PARITY}(x_{1:i}) = x_i \oplus {t}_{i - 1}$ (${t}_0 = 0$).
The visual dependency between sharpness and length of input for PARITY with a scratchpad is shown in Figure \ref{exp:scratchpad-sharpness}. Even for length around 300, sharpness is low and there is little increase with input length. 
Thus, decrease in sensitivity due to the scratchpad  can explain why prior work \citep{anil2022exploring} found that PARITY is easy for Transformers with scratchpad.

\paragraph{LayerNorm Blowup Enables Learning}
Figure \ref{exp:dynamic-small} represents the evolution of loss and sharpness of Transformer models trained for PARITY with input length 25. The results are averaged across 39 converged runs. Similar curves for parameter norm and LayerNorm blowup are presented in Appendix \ref{exp:dynamic-main}.

A dramatic increase in sharpness occurs at exactly the time when loss falls to 0. 
This suggests the presence of a steep minimum in the loss landscape, and fitting PARITY requires the optimization procedure to find this minimum.
Figure \ref{exp:dynamic-main} (Appendix) shows further details of this process. During learning, there is a small but sharp increase in the blowup which makes the high sensitivity of the model possible, increasing the left-hand side of the inequality in Corollary \ref{thm:bigtheorem}. Following that, non-zero weight decay drives the parameter norm down, which -- by the theoretically predicted tradeoff between blowup and parameter norm -- is accompanied by an exponential increase in blowup.

\section{Discussion}
We have provided a rigorous theoretical explanation of the inductive bias towards low sensitivity observed in empirical research on transformers.
Theorem~\ref{thm:lrho-bound} describes a fundamental inductive bias of the transformer architecture: for long inputs, fitting sensitive functions is only possible in sharp minima of the loss.
This holds without assumptions often made in previous work about expressive capacity, such as non-infinite precision \citep[e.g.][]{merrill2023logic, angluin2023masked}, hard attention \citep[e.g.][]{hahn2020theoretical}, or Lipschitz-continuous variants of layer norm \citep{edelman2022inductive}.
We speculate that Theorem~\ref{thm:lrho-bound} reflects a fundamental limitation of parallelized differentiable computing with bounded depth.
Our results show that it is overcome by scaling the number of computation steps with the input length.
While we focused on functions outputting a single label; an interesting direction for future research is the extension of this theory to sequence-to-sequence transductions, for which some empirical data has been reported \citep[e.g.][]{zhou2023algorithms}, but theoretical understanding remains wide open.

Due to the relation between average sensitivity and Fourier analysis on the Boolean cube (Equation~\ref{eq:sensitivity-fourier}), a low-sensitivity bias can be viewed as a sequence modeling analogue of a spectral bias towards low frequencies observed in other neural architectures \citep[e.g.][]{pmlr-v97-rahaman19a, NEURIPS2022_306264db}.

We note that, while our results show that sensitive transformers are very brittle, these results do not by themselves have implications for \emph{real-world generalization}, as a low-sensitivity bias need not always be beneficial in real-world setups. Indeed, the relationship between sharpness and real-world generalization is not straightforward \citep[e.g.][]{andriushchenko2023modern,kaur2023maximum}. Our theory suggests that transformers generalize well to the extent that real-world data has bounded sensitivity \citep[e.g.][]{hahn2020sensitivity}.

\section{Conclusion}

We have proven that, under the transformer architecture, high sensitivity in input space can only be achieved in very sharp minima.
Empirical results confirm the predictions of the theory.
Taken together, our results explain a diverse set of empirical observations about transformers not explained by previous theoretical work.
They suggest shifting theoretical research from in-principle expressiveness considerations to studying quantitative bounds and the shape of the loss landscape in order to understand the abilities of transformers.

\section*{Limitations}

A limitation of our results is that the theoretical results are asymptotic, providing statements about the limit of very long input sequences. Providing more quantitative bounds that tightly characterize finite-length behavior is an interesting problem for future research.

A second limitation is that we only target functions outputting a single value. Operationalizing sensitivity-like metrics for sequence-to-sequence functions may be required in order to expand the theory to such functions.

Third, our results apply to transformer encoders. It remains open if transformer decoders, with causal attention masking, face a different set of limitations than we have shown here in the absence of masking.

Fourth, our theoretical results concern the loss landscape, not the training dynamics itself. Further technical advances may be needed to directly prove corresponding results for training dynamics.

\section*{Acknowledgments}
We thank Lena Strobl, Dana Angluin, and David Chiang for useful discussion.
We thank David Chiang, Paul Lintilhac, and Yuval Pinter for spotting various errors in a previous version, and the anonymous ARR reviewers for useful feedback.

\bibliography{literature}

\newpage\onecolumn
\appendix


\section{Simple Constructions for PARITY}\label{sec:appendix:construct-parity}

While it is well-known that PARITY can be expressed using uniform circuits with majority gates \citep{haastad1986computational}, it may not be obvious just how simple the construction can be, and how it can easily embed into formalisms like RASP.
In order to highlight the existence of very simple constructions for PARITY in various formalisms, and to make discussion here self-contained, we provide a simple proof.
The key is the following insight:
\begin{lemma}[Folklore]\label{lemma:folklore-parity}
A bit string is odd if and only if ($\dagger$) for the (strict) majority of positions $i$ holding a 1, it holds that strictly more than 1/2 of 1s appear no later than at $i$.
\end{lemma}
This fact appears to be well known in the circuit complexity community, though we are not aware of the original published reference. We provide a proof for self-containedness.
\begin{proof}
As only positions holding a one play a role, it is sufficient to consider strings of the form $x \in 1^*$.
First, consider the case where $x = 1^{2n}$:
\begin{center}
\begin{tabular}{c|ccccccccc}
Index & 1 & 2 & \dots & n & n+1 & \dots & 2n \\ \hline
A: \# of 1s no later than at $i$ & 1 & 2 & $\dots$ & n & n+1 & $\dots$ & 2n \\
B: \# of 1s later than at $i$ & 2n-1 & 2n-2 & $\dots$ & n & n-1 & $\dots$ & 0\\
A>B? & no & no & $\dots$ & no & yes & $\dots$ & yes
\end{tabular}
\end{center}
Here, $A>B$ holds at exactly $n$ positions, less than the strict majority of positions.
On the other hand, when $x=1^{2n+1}$:
\begin{center}
\begin{tabular}{c|ccccccccc}
Index & 1 & 2 & \dots & n & n+1 & \dots & 2n & 2n+1  \\ \hline
A: \# of 1s no later than at $i$ & 1 & 2 & $\dots$ & n & n+1 & $\dots$ & 2n & 2n+1 \\
B: \# of 1s later than at $i$ & 2n & 2n-1 & $\dots$ & n+1 & n & $\dots$ & 1 & 0\\
A>B? & no & no & $\dots$ & no & yes & $\dots$ & yes & yes
\end{tabular}
\end{center}
Here, $A>B$ holds at exactly $n+1$ positions, a strict majority of positions.
\end{proof}

\begin{proof}[Proof of Fact 1]
    The property ($\dagger$) from Lemma~\ref{lemma:folklore-parity} is straightforwardly formalized in RASP.
To formalize it in FO[M] \citep{merrill2023logic}, we note that FO[M] permits defining constructs of the form ``for the majority of positions $x$ satisfying $\phi(x)$, it holds that $\psi(x)$'', for any formula $\psi$.
    FO[M] formulas have a known translation to (highly uniform) majority circuits.
    This concludes the proof of Fact 1.
    
\end{proof}

Lemma~\ref{lemma:folklore-parity} shows that, in terms of expressive capacity, PARITY is tightly linked to MAJORITY.
While the components of ($\dagger$), MAJORITY and position comparison, are easily learned by transformers, the composite function is hard to learn.
This observation suggests that function classes satisfying typical closure properties, as satisfied by typical logic and circuit classes, cannot model which functions transformers learn easily.
On the other hand, average sensitivity elegantly explains this difference: The expression ($\dagger$) uses two nested ``majority'' operations, each introducing an average sensitivity of $\sqrt{\seqlen}$, multiplying to the asymptotically  larger sensitivity $n$ of the composite function.

This result does not exclude a role of minimizing description length in formalisms such as RASP as a factor impacting transformers' inductive biases (as suggested empirically by \citet{zhou2023algorithms}), but suggests that description length and expressive capacity might be unable to account for important aspects of transformers' inductive biases, including the low-sensitivity bias and difficulty learning PARITY.

\section{Proof of Theorem 4 and Corollary 5}\label{sec:proof}
For a bitstring $x \in \{\pm 1\}^n$, we write $\influence_i[f](x)$ to denote the \emph{absolute influence} of the $i$-th bit: 
\begin{equation}\label{eq:def:influence}
    \influence_i[f](x) := \left\|f(x) - f(x\flipBit{i})\right\|_2 = \sqrt{\sum_s |f(x)_s - f(x\flipBit{i})_s|^2}
\end{equation}
which simplifies to the absolute value $|f(x) - f(x\flipBit{i})|$ if $f(x) \in \mathbb{R}$.
We will suppress the argument $x$ where it is contextually given.
When $f$ outputs a scalar, then we have by (\ref{eq:def:s}):
\begin{equation}\label{eq:sens-infl-link}
    s(f,x) = \frac{1}{4} \sum_{i=1}^\seqlen \influence_i[f](x)^2
\end{equation}
We begin by noting
\begin{equation}
\|f(x)-f(x\flipBit{q})\|_2^2 \leq \|f(x)-f(x\flipBit{q})\|_2 \cdot \max_x \|f(x)-f(x\flipBit{q})\|_2 \leq 2 \max_x \|f(x)\|_2 \cdot  \|f(x)-f(x\flipBit{q})\|_2
\end{equation}
and hence
\begin{equation}\label{eq:l1-inf-bounds-l2-inf}
\frac{1}{2}\influence_q[f]^2(x) = \frac{1}{2} \|f(x)-f(x\flipBit{q})\|_2^2 \leq  (\max_x \|f(x)\|_2) \|f(x)-f(x\flipBit{q})\|_2 =  (\max_{x'} \|f(x')\|_2) \influence_q[f](x)
\end{equation}

\subsection{Bounding the Sensitivity of the Attention Heads}\label{sec:sensitivity-attention}

The first step will be bound the sensitivity of the outputs of individual attention heads.
At a high level, such bounds are also part of the existing pointwise Lipschitzness bounds \citep{hahn2020theoretical, edelman2022inductive, li2023stability}.
Our key innovation in this part of the proof is to replace a \emph{pointwise} bound with a \emph{high-probability} bound that, at the input layer, is independent of the key-query matrix.
The key novel part is Lemma~\ref{lemma:attention-bound-first-layer}.
In order to make the proof self-contained, we also include proofs for the other steps.

We will refer to the following bounds, for suitably defined constants $1 \leq C_V, C_P, C_A^{(k)}, C_{MLP}, L_{f^{MLP}}, C^{(k)}, C < \infty$:
\begin{align}
    \|V_{k,h}\|_{2} \leq &C_V \\
    \left(\max_w \|y_w^{(k)}\|_2\right) \leq &\sqrt{d} \text{ when $k\geq 1$}\\
    \left(\max_w \|y_w^{(0)}\|_2\right) \leq &C_P \label{app:eq:cp}\\
    a_{ij}^{(k,h)} \leq &C_A^{(k)} := d \cdot \max_h \|K^T_{k,h}Q_{k,h}\|_{spectral} \ \ \ \ \ \ \ \ \ \text{ if $k>1$}\label{eq:def-ca-ahat-bound-ca} \\
    C_{MLP} := &\sup_{x : \|x\|_2 \leq (H+1)C_V\min\{C_P, \sqrt{d}\}} \|f^{MLP}(x)\|_2\label{app:eq:cmlp} \\
    L_{f^{MLP}} := &  \sup_{\|x\|_2, \|x'\|_2\leq (H+1)\sqrt{d}} \frac{\|f^{MLP}(x)-f^{MLP}(x')\|_2}{\|x-x'\|_2}  \label{app:eq:lfmlp} \\
   C^{(0)} :=& 150 L_{f^{MLP}} C_V^2 C_P H  \label{eq:def-c0} \\
    C^{(1)} = &2 \left( \max_w \|\ypre_w^{(1)}\|_{2}\right)   \cdot L_{f^{MLP}} \leq 2 C_{MLP} (H+1)  C_V C_P   \cdot L_{f^{MLP}} 
\\
\label{eq:def-ck}
    C^{(k)} :=& 12 L_{f^{MLP}} (H+1)  \|V_{k,h}\|_{1} 4 \sqrt{d}   \exp(4 C_A^{(k)})  \ \ \ \ \ \ \text{       (for $k>1$)} \\
\|\ypre_j^{(1)}\|_2 \leq^{(\ref{eq:def:ypre})} & C_{MLP} (H+1)  C_V C_P \\
C := & 25 \sqrt{d} H (1+\|v_{out}\|)  \left(\prod_{l=0}^k C^{(k)}\right)  \label{eq:def:c-appendix}
\end{align}
We first bound the influence of any bit on an attention head's activation in terms of its influence on the attention values and the activations at the preceding layer:
\begin{lemma}[Sensitivity of an Attention Head]\label{lemma:head-influence-bound}
We have
    \begin{equation}\label{eq:lemma:head-influence-bound}
\influence_q[\headmac{j}{h}{k}] \leq  C_V \sum_{w=1}^{\seqlen}\hat{a}_{j,w}^{k,h} \influence_q[y_w^{(k-1)}] + C_V \left(\max_w \|y_w^{(k-1)}\|_2\right)  \sum_{w=1}^{\seqlen} \influence_q[\hat{a}_{j,w}^{k,h}] 
\end{equation}

\end{lemma}

The main task will be to bound $\sum_{w=1}^{\seqlen} \influence_q[\hat{a}_{j,w}^{k,h}]$.
We separately bound this for $k=1$ and $k>1$.

\begin{remark}\label{remark:comparison-head-edelman}
We note that the conceptually similar Lemma 4.3 in \citet{edelman2022inductive} might suggest a seemingly stronger bound where
$
    \sum_w \influence_q[\widehat{a}_{j,w}^{k,h}]
$
is replaced (using their Lemma A.6) by the \emph{maximum} influence on the \emph{logits}:
$
   \max_{w=1,...,N}  \influence_q[a_{j,w}^{k,h}]
$
using the fact that softmax has a Lipschitz constant $\leq 1$.
However, $s(f,x)$ requires summing over $q$ (not relevant to \citet{edelman2022inductive}\footnote{They considered Lipschitzness in the parameter space, with parameters that grow at most sublinearly with $n$. In contrast, we investigate sensitivity in the input space.}), and the expression
$
    \sum_q \max_{w=1,...,N}  \influence_q[a_{j,w}^{k,h}]
$
can easily be $\Theta(n)$, for instance, if $a_{j,w}^{1,h} \equiv 1_{x_w = 1}$.
Importantly, at the lowest layer, we find sublinear bounds for
$
    \sum_q \sum_w \influence_q[\widehat{a}_{j,w}^{k,h}]
$
that, at $k=1$, involve no bound on the logits $a_{j,w}^{k,h}$ at all.
\end{remark}

\begin{proof}[Proof of Lemma~\ref{lemma:head-influence-bound}]
Using the definition of $\head$:
\begin{equation}
    \headmac{i}{h}{k} = \sum_{j=1}^{\seqlen}\hat{a}_{i,j}^{(k,h)} V_{k,h} y_j^{(k-1)} 
\end{equation}
and the triangle inequality, we find:
\begin{align*}
\influence_q[\headmac{j}{h}{k}] = &
\left\|\sum_{j=1}^{\seqlen}\hat{a}_{i,j}^{(k,h)}(x) \cdot V_{k,h} \cdot  y_j^{(k-1)}(x) - \sum_{j=1}^{\seqlen}\hat{a}_{i,j}^{(k,h)}(x\flipBit{q}) \cdot V_{k,h} \cdot y_j^{(k-1)}(x\flipBit{q})\right\|_2
\\
\leq &
\sum_{j=1}^{\seqlen}\left\|\hat{a}_{i,j}^{(k,h)}(x) V_{k,h} y_j^{(k-1)}(x) - \hat{a}_{i,j}^{(k,h)}(x\flipBit{q}) V_{k,h} y_j^{(k-1)}(x\flipBit{q})\right\|_2
\\
\leq &
\sum_{j=1}^{\seqlen}|\hat{a}_{i,j}^{(k,h)}(x)|\left\| V_{k,h} y_j^{(k-1)}(x) -  V_{k,h} y_j^{(k-1)}(x\flipBit{q})\right\|_2
+ \left|\hat{a}_{i,j}^{(k,h)}(x)  - \hat{a}_{i,j}^{(k,h)}(x\flipBit{q}) \right|
\left\| V_{k,h} y_j^{(k-1)}(x)\right\|_2
\\
\leq & C_V \sum_{w=1}^{\seqlen}\hat{a}_{j,w}^{k,h} \influence_q[y_w^{(k-1)}] +  C_V\sum_{w=1}^{\seqlen}\|y_w^{(k-1)} \|_2 \influence_q[\hat{a}_{j,w}^{k,h}] \\
\leq & C_V  \sum_{w=1}^{\seqlen}\hat{a}_{j,w}^{k,h} \influence_q[y_w^{(k-1)}] + C_V  \left(\max_w \|y_w^{(k-1)}\|_2\right)  \sum_{w=1}^{\seqlen} \influence_q[\hat{a}_{j,w}^{k,h}] \\
\end{align*}
\end{proof}

\paragraph{Bounding the Sensitivity of Attention (First Layer)}
First, we bound the sensitivity of the attention distribution in the first layer.
Here, it will be crucial to provide high-probability bounds independent of $\|K_{1,h}^T Q_{1,h}\|$, unlike the pointwise Lipschitzness bounds of \citet{hahn2020theoretical, li2023stability}.
The resulting Lemma~\ref{lemma:attention-bound-first-layer} is one of the key innovations compared to the pointwise Lipschitzness bounds from prior work, and the main technical innovation of Section~\ref{sec:sensitivity-attention}.
\begin{lemma}\label{lemma:attention-bound-first-layer}
Let $\delta>0$.
With probability $\geq 1-\frac{H}{\seqlen^{\delta-2}}$ over the choice of $x$, the following holds:
For each $i = 1, \dots, \seqlen$, $h=1, \dots, H$,
\begin{align}\label{eq:lemma:attention-bound-first-layer}
    \sum_{q=1}^{\seqlen}\sum_{j=1}^{\seqlen}\influence_q[\weightmac{i}{j}{1}{h}] \leq & (8+32\delta)\log{\seqlen} + 10
\end{align}
\end{lemma}

\begin{proof}[Proof of Lemma~\ref{lemma:attention-bound-first-layer}]
Let $x_i$ denote the token at position $i$ (-1 or 1), and $e(x_i)$ the corresponding word embedding.

Throughout, we will suppress the index $h = 1, \dots, H$ for notational simplicity.

If $e(x_i)$ is the word embedding for the token $x_i$ and $p_i$ is the $i$-th positional embedding, then:
\begin{align*}
    a^{(1)}_{ij} = & \left(p_i^T + e(x_i)^T \right) \cdot Q^T K \cdot \left(p_j+e(x_j)\right) \\
    = & \left(\begin{matrix}p_i^T & e(x_i)^T \end{matrix}\right) \left(\begin{matrix} I_d & I_d\end{matrix}\right)^T Q^T K \left(\begin{matrix} I_d & I_d\end{matrix}\right) \left(\begin{matrix}p_j \\ e(x_j) \end{matrix}\right) \\
    = & \left(\begin{matrix}p_i^T & e(x_i)^T \end{matrix}\right) \left(\begin{matrix} \logMatPosPos & \logMatTokPos \\ \logMatPosTok & \logMatTokTok \end{matrix}\right) \left(\begin{matrix}p_j \\ e(x_j) \end{matrix}\right) \\
    =& \underbrace{p_i^T\logMatPosPos p_j}_{\logMatPosPos^{(ij)}} + \underbrace{p_i^T \logMatPosTok}_{\logMatPosTok_i} e(x_j) + e(x_i)^T \underbrace{\logMatTokPos p_j}_{\logMatTokPos_j} + e(x_i)^T\logMatTokTok e(x_j)
\end{align*}
for appropriate matrices $\logMatPosPos, \logMatTokPos, \logMatPosTok, \logMatTokTok$.
Note that this decomposition holds for any linear combination of 
Note that \emph{adding} positional and word embeddings is a special case, where $Q$ and $K$ can each be written as products of two matrices, where the second one has the form $\left(\begin{matrix} I_d & I_d\end{matrix}\right)$.
Then
\begin{align*}
    \weightmac{i}{j}{1}{h} = & \frac{\exp(q_i(x_i) k_j(x_j))}{\sum_{w=1}^{\seqlen} \exp(q_i(x_i) k_w(x_w))} \\
     = & \frac{\exp(\logMatPosTok_i e(x_j) + e(x_i)^T \logMatTokPos_j + e(x_j)\logMatTokTok e(x_i)+A^{(ij)} )}{\sum_{w=1}^{\seqlen} \exp(Q_i e(x_w) + e(x_i)^T \logMatTokPos_w + e(x_w)\logMatTokTok e(x_i)+\logMatPosPos^{(iw)})} \\
\end{align*}
We can write, taking $\widehat{x}_i \in \{0,1\}$ to be the indicator that $x_i=1$\footnote{I.e., $\widehat{x}_i = 1$ of $x_i=1$ and else 0.}:
\begin{align*}
    & \exp(e(x_i)^T {\logMatTokPos}_w + \logMatPosPos^{(iw)} + {\logMatPosTok}_i e(x_w) + e(x_i)^T{\logMatTokTok}e(x_w)) \\
        = & {\exp\left(e(x_i)^T {\logMatTokPos}_w + \logMatPosPos^{(iw)}\right)} \cdot {\exp\left( {\logMatPosTok}_i e(x_w) + e(x_i)^T{\logMatTokTok}e(x_w)\right)} \\
    = & \underbrace{\exp(e(x_i)^T {\logMatTokPos}_w + \logMatPosPos^{(iw)}  +  {\logMatPosTok}_i e(-1) + e(x_i)^T{\logMatTokTok}e(-1)) }_{\gamma_{iw}} \underbrace{\exp( {\logMatPosTok}_i e(x_w) + e(x_i)^T{\logMatTokTok}e(x_w) -  {\logMatPosTok}_i e(-1) - e(x_i)^T{\logMatTokTok}e(-1))}_{1+\widehat{x}_w\rho_i} \\
\end{align*}
where
\begin{align*}
   \rho_i = & \exp( \logMatPosTok_i e(1) + e(x_i)^T{\logMatTokTok}e(1) -  \logMatPosTok_i e(-1) - e(x_i)^T\logMatTokTok e(-1)) - 1
\end{align*}
Next, we may assume w.l.o.g. that $\rho_i$ is nonnegative, by, if $\rho_i$ otherwise were negative, renaming the input bits across the entire input string for a fixed $i$.
So we have
\begin{align*}
    \weightmac{i}{j}{1}{h} = \frac{\gamma_{ij}(1+\widehat{x}_j\rho_i)}{\sum_w \gamma_{iw} (1+\widehat{x}_w\rho_i)}
\end{align*}
We split the summation over $j$ and $q$ into three cases: $q\neq i,j$; $q=j, q\neq i$; $q=i$:
\begin{equation}
    \sum_{j=1}^{\seqlen}\sum_{q=1}^{\seqlen} = \sum_{j,q: q\neq i,j} + \sum_{j,q : q=j, q\neq i} + \sum_{j,q: q=i}
\end{equation}
First, for $q\neq i,j$:
\begin{align*}
    \influence_q[\weightmac{i}{j}{1}{h}] =& \frac{\gamma_{ij}(1+\widehat{x}_j\rho_i)}{\gamma_{iq} + \sum_{w\neq q} \gamma_{iw} (1+\widehat{x}_w\rho_i)} - \frac{\gamma_{ij}(1+\widehat{x}_j\rho_i)}{\gamma_{iq} (1+\rho_i) + \sum_{w\neq q} \gamma_{iw} (1+\widehat{x}_w\rho_i)} \\
    =& \frac{\gamma_{ij}(1+\widehat{x}_j\rho_i) \gamma_{iq}\rho_i}{(\gamma_{iq} + \sum_{w\neq q} \gamma_{iw} (1+\widehat{x}_w\rho_i))(\gamma_{iq} (1+\rho_i) + \sum_{w\neq q} \gamma_{iw} (1+\widehat{x}_w\rho_i))} \\
    \leq & \frac{\gamma_{ij}(1+\widehat{x}_j\rho_i) \gamma_{iq}(1+\rho_i)}{(\gamma_{iq} + \sum_{w\neq q} \gamma_{iw} (1+\widehat{x}_w\rho_i))(\gamma_{iq} (1+\rho_i) + \sum_{w\neq q} \gamma_{iw} (1+\widehat{x}_w\rho_i))} \\
    = & \frac{\gamma_{ij}(1+\widehat{x}_j\rho_i) }{(\gamma_{iq} + \sum_{w\neq q} \gamma_{iw} (1+\widehat{x}_w\rho_i))}
    \frac{\gamma_{iq}(1+\rho_i)}{(\gamma_{iq} (1+\rho_i) + \sum_{w\neq q} \gamma_{iw} (1+\widehat{x}_w\rho_i))}\\
    = & \weightmac{i}{j}{1}{h}(x : \widehat{x}_q=0)\weightmac{i}{q}{1}{h}(x:\widehat{x}_q=1)  
\end{align*}
where $\weightmac{i}{j}{1}{h}(x : \widehat{x}_q=0)$ is $\weightmac{i}{j}{1}{h}$ evaluated at the input $x = x_1\dots x_{q-1} (-1) x_{q+1} \dots x_\seqlen$, and similarly for $\weightmac{i}{q}{1}{h}(x:\widehat{x}_q=1)$.

We may assume w.l.o.g., by appropriate reordering given a fixed $i$, that $\gamma_{i1} \geq \gamma_{i2} \geq \gamma_{i3} \geq \dots$.
Then,
\begin{align*}
\sum_q\sum_{j\neq q} \influence_q[\weightmac{i}{j}{1}{h}]    \leq&
\sum_q\sum_{j\neq q} \weightmac{i}{j}{1}{h}(x : \widehat{x}_q=0)\weightmac{i}{q}{1}{h}(x:\widehat{x}_q=1)
\\
= & \sum_q\weightmac{i}{q}{1}{h}(x:\widehat{x}_q=1) \sum_{j\neq q} \weightmac{i}{j}{1}{h}(x : \widehat{x}_q=0)
\\
\leq & \sum_q \weightmac{i}{q}{1}{h}(x:\widehat{x}_q=1) \\
    = & \sum_q \frac{\gamma_{iq}(1+\rho_i)}{\gamma_{iq}(1+\rho_i) + \sum_{w\neq q} \gamma_{iw} (1+\widehat{x}_w\rho_i)} \\
    \leq^{(*)} & \sum_{q>16{\delta}\log{\seqlen}} \frac{\gamma_{iq}(1+\rho_i)}{\gamma_{iq}(1+\rho_i) + \gamma_{iq} \sum_{w<q}  (1+\frac{1}{4}\rho_i)} + \sum_{q\leq 16{\delta}\log{\seqlen}} \frac{\gamma_{iq}(1+\rho_i)}{\gamma_{iq}(1+\rho_i)} \\
    \leq & \sum_{q>16{\delta}\log{\seqlen}} \frac{(1+\rho_i)}{(1+\rho_i) +  \sum_{w<q}  (1+\frac{1}{4}\rho_i)} + 16{\delta}\log{\seqlen} \\
\leq & \sum_{q>16{\delta}\log{\seqlen}} \frac{(1+\rho_i)}{(1+\rho_i) +  q  (\frac{1+\rho_i}{4})} + 16{\delta}\log{\seqlen} \\
= & \sum_{q>16{\delta}\log{\seqlen}} \frac{4}{   4+q} + 16{\delta}\log{\seqlen} \\
\leq & 4 \sum_{q>0} \frac{1}{   q} + 16{\delta}\log{\seqlen} \\
\leq & (4+16{\delta})\log{\seqlen} +4\\
\end{align*}
where $(*)$ holds with overwhelming probability over the choice of $x$, say, $1-\frac{1}{\seqlen^{{\delta}}}$ (multiplicative Chernoff bound for binomial random variables), for any individual $q$, and $1-\frac{1}{\seqlen^{{\delta-1}}}$ for the whole expression by a union bound over all $q$.
We further used $\sum_{i=1}^{\seqlen}\frac{1}{i} \leq \log{\seqlen} + 1$.

Second, consider the case where $q=j$, $q\neq i$:
\begin{align*}
  \sum_{q=j; q\neq i}  \influence_j[\weightmac{i}{j}{1}{h}] =& \sum_{q=j; q\neq i; q > 16\delta\log n}  \influence_j[\weightmac{i}{j}{1}{h}] + \sum_{q=j; q\neq i; q \leq 16\delta\log n}  \influence_j[\weightmac{i}{j}{1}{h}]\\
\leq & \sum_{q=j; q\neq i; q > 16\delta\log n}  \influence_j[\weightmac{i}{j}{1}{h}] +16\delta\log n 
\end{align*}
Then:
\begin{align*}
  \sum_{q=j; q\neq i; q > 16\delta\log n}  \influence_j[\weightmac{i}{j}{1}{h}] =& \sum_{q=j; q\neq i} \frac{\gamma_{ij}(1+1\rho_i)}{\gamma_{ij} (1+\rho_i) + \sum_{w\neq j} \gamma_{iw} (1+\widehat{x}_w\rho_i)} - \frac{\gamma_{ij}}{\gamma_{ij} + \sum_{w\neq j} \gamma_{iw} (1+\widehat{x}_w\rho_i)}\\
    =& \sum_{q=j; q\neq i; q > 16\delta\log n}\frac{
     \gamma_{ij}\rho_i \sum_{w\neq j} \gamma_{iw} (1+\widehat{x}_w\rho_i)
    }{(\gamma_{ij} + \sum_{w\neq j} \gamma_{iw} (1+\widehat{x}_w\rho_i))(\gamma_{ij} + \sum_{w\neq j} \gamma_{iw} (1+\widehat{x}_w\rho_i) + \gamma_{ij} \rho_i + \sum_{w\neq j} \gamma_{iw} (1+\widehat{x}_w\rho_i))} \\
    \leq & \sum_{q=j; q\neq i; q > 16\delta\log n}\frac{
     \gamma_{ij}\rho_i \cdot \left(\gamma_{ij} + \sum_{w\neq j} \gamma_{iw} (1+\widehat{x}_w\rho_i) + \gamma_{ij} \rho_i + \sum_{w\neq j} \gamma_{iw} (1+\widehat{x}_w\rho_i)\right)}{(\gamma_{ij} + \sum_{w\neq j} \gamma_{iw} (1+\widehat{x}_w\rho_i))(\gamma_{ij} + \sum_{w\neq j} \gamma_{iw} (1+\widehat{x}_w\rho_i) + \gamma_{ij} \rho_i + \sum_{w\neq j} \gamma_{iw} (1+\widehat{x}_w\rho_i))} \\
    =& \sum_{q=j; q\neq i; q > 16\delta\log n}\frac{
     \gamma_{ij}\rho_i 
    }{(\gamma_{ij} + \sum_{w\neq j} \gamma_{iw} (1+\widehat{x}_w\rho_i) )} \\
    \leq& \sum_{q=j; q\neq i; q > 16\delta\log n}\frac{
     \gamma_{ij}\rho_i 
    }{(\gamma_{ij} + \sum_{w < j} \gamma_{iw} (1+\widehat{x}_w\rho_i) )} \\
    \leq& \sum_{q=j; q\neq i; q > 16\delta\log n}\frac{
     \gamma_{ij}\rho_i 
    }{(\gamma_{ij} + \gamma_{ij} \sum_{w < j}  (1+\widehat{x}_w\rho_i) )} \\
    \leq^{(*)} & \sum_{q=j; q\neq i}\frac{
     \gamma_{ij}\rho_i 
    }{(\gamma_{ij} + \gamma_{ij} j  \frac{1+\rho_i}{4}) )} \\
    =& \sum_{q=j; q\neq i}\frac{
     \rho_i 
    }{(1 +  j  \frac{1+\rho_i}{4}) )} \\
    \leq& \sum_{q=j; q\neq i}\frac{
     4\rho_i 
    }{4 +  j  +j\rho_i} \\
    \leq&  \sum_{q=j; q\neq i}\frac{
     4\rho_i 
    }{ j+j\rho_i } \\
    =& \frac{
    4 \rho_i 
    }{  1+\rho_i } \sum_{q=j; q\neq i}\frac{
     1 
    }{ j } \\
    \leq& 4 \log{\seqlen} + 4 \\
\end{align*}
where again $(*)$ holds with probability $1-\frac{1}{\seqlen^{\delta-1}}$ by a Chernoff bound for each $j$, and a union bound across all $j$. Importantly, the third inequality used the fact that $\gamma_{iw} \geq \gamma_{ij}$ when $w<j$.
So overall
\begin{align*}
	\sum_{q=j; q\neq i}  \influence_j[\weightmac{i}{j}{1}{h}] \leq 4 \log{\seqlen} + 4 + 16\delta\log{\seqlen}
\end{align*}
with probability $1-\frac{1}{\seqlen^{\delta-1}}$.

Third, for the case $q=i$:
\begin{align*}
    \sum_{j\neq q} \influence_q[\weightmac{q}{j}{1}{h}] = &   \sum_{j\neq q} |\weightmac{q}{j}{1}{h}(x)-\weightmac{q}{j}{1}{h}(x\flipBit{q})|\\
  \leq &   \sum_{j\neq q} |\weightmac{q}{j}{1}{h}(x)| + \sum_{j\neq q} |\weightmac{q}{j}{1}{h}(x\flipBit{q})|\\
  \leq & 2 
\end{align*}

Overall, we get
\begin{align}\label{eq:double-sum-influence-att-layer1}
    \sum_{q=1}^{\seqlen}\sum_{j=1}^{\seqlen}\influence_q[\weightmac{i}{j}{1}{h}] \leq & (8+32\delta)\log{\seqlen} + 10
\end{align}
for each $i, h$,
with probability $1-\frac{H}{\seqlen^{\delta-2}}$, via a union bound.
\end{proof}

\begin{remark}
To illustrate why the statement can only hold with high probability, rather than pointwise at every $x$, we consider a one-layer transformer where the first layer at each position computes the OR function, accomplished by setting
\begin{equation}
    a_{ij}^{(1)} = \begin{cases}
        \log{\seqlen} & \text{if }x_j = 1 \\
        -\log{\seqlen} & \text{if }x_j = 0
    \end{cases}
\end{equation}
That is, any occurrence of a 1 draws attention to it, allowing a subsequent MLP to test whether a 1 appears in the string.
Consider the input $x=0^n$.
Here, for any Hamming neighbor $x' = x\flipBit{q}$:
\begin{equation}
    \weightmac{i}{j}{1}{h}  = \begin{cases}
        \frac{n}{\seqlen+\frac{\seqlen-1}{n}} = \frac{1}{1+1/\seqlen-1/\seqlen^2} \geq 1-\frac{1}{\seqlen} & \text{if }j = q \\
        \frac{1/n}{\seqlen+\frac{\seqlen-1}{n}} = \frac{1}{\seqlen^2+\seqlen-1}\leq \frac{1}{\seqlen^2} & \text{else}
    \end{cases}
\end{equation}
Hence, 
\begin{equation}
    \influence_q[\weightmac{i}{j}{1}{h} ] \geq \begin{cases}
        1-\frac{2}{n} & \text{if }j=q \\
        \frac{1}{\seqlen} - \frac{1}{\seqlen^2} & else
    \end{cases}
\end{equation}
and
\begin{align}
    \sum_{q=1}^{\seqlen}\sum_{j=1}^{\seqlen}\influence_q[\weightmac{i}{j}{1}{h}] \geq \sum_{q=1}^{\seqlen}\influence_q[\weightmac{i}{q}{1}{h}] \geq \seqlen-2
\end{align}
The bound (\ref{eq:lemma:attention-bound-first-layer}) thus cannot hold for all input strings $x$.
However, the lemma shows that it holds for the vast majority of inputs.
The intuitive reason is that the vast majority of strings $x \in \{0,1\}^n$ have a much less skewed distribution of ones and zeros. For a string where a substantial number of both ones and zeros appear, the influence of any individual input bit is low; the proof of Lemma~\ref{lemma:attention-bound-first-layer}  makes this idea rigorous using a simple concentration bound argument.
This intuition roughly matches the sensitivity properties of the OR function: $s(f_{OR}, 0^n) =n$, but $as_\seqlen(f_{OR}) = o(1)$, because for the vast majority of strings, flipping any bit cannot flip the OR output.
Lemma~\ref{lemma:attention-bound-first-layer} generalizes this idea to arbitrary assignments of attention logits, and to the continuous output of the softmax operation, as opposed to discrete Boolean functions such as OR.
\end{remark}

\paragraph{Bounding the Sensitivity of Attention: Higher Layers}
At higher layers, the logits $a_{ij}$ may be highly correlated across $j$, impeding the use of concentration bounds. 
In fact, for our purposes here, a pointwise Lipschitzness bound as in \citet{hahn2020theoretical, li2023stability} will be sufficient here (roughly corresponding to Lemmas B.1, B.2 of \citet{li2023stability}). While no novelty is claimed for Lemma~\ref{lemma:influence-att-higher}, we include its proof for completeness.

\begin{lemma}\label{lemma:influence-att-higher}
For $k>1$ and $q \in \{1,\dots,n\}$,
\begin{equation}\label{eq:higher-layers-att-influence}
\influence_q[\weightmac{i}{j}{k}{h}] \leq  2  \exp(4C_A^{(k)})  
	\left( \influence_q[\ypost_i^{(k-1)}] + \frac{1}{\seqlen} \sum_{j=1}^{\seqlen}  \influence_q[\ypost_j^{(k-1)}] \right)
\end{equation}
\end{lemma}

\begin{remark}
Note that this bound depends on $\|K^TQ\|_{spectral}$ , unlike the first layer, where we found a bound independent of $\|K^TQ\|_{spectral}$.
    We do not know whether the exponential dependency on $\|K^TQ\|_{spectral}$ is optimal.
    However, some dependency on $\|K^TQ\|_{spectral}$ is unavoidable.
    To see why, consider a transformer where $\ypost_\seqlen^{(\numLayers)}$ computes a function close to $\frac{PARITY}{\seqlen}$ at $\epsilon=1$ \citep{chiang2022overcoming}, and $y_{1}^{(\numLayers)}$ computes its negation.
    Then adding another layer that, by rescaling logits with a factor of $\seqlen^2$, attends to $\ypost_{1}^{(\numLayers)}$ or $\ypost_{N}^{(\seqlen)}$ depending on the input's parity, can compute a function very close to PARITY, if the two activations also provide some disambiguating positional information. This holds even for high-probability or on-average bounds. However, Lemma~\ref{lemma:attention-bound-first-layer} shows that such dependency can be eliminated at the lowest layer.
\end{remark}
\begin{proof}

Fixing $k,h,i$, we write
\begin{align}\label{eq:cu-bound}
c_u =& \exp\left(a_{iu}^{(k,h)}(x)\right) \leq^{(\ref{eq:def-ca-ahat-bound-ca})} \exp(C_A^{(k)}) \\
d_u =& \exp\left(a_{iu}^{(k,h)}(x)\right) - \exp\left(a_{iu}^{(k,h)}(x\flipBit{q})\right) \leq^{(\ref{eq:def-ca-ahat-bound-ca})} 2\exp(C_A^{(k)})
\end{align}
Then:
\begin{align*}
   \left| \weightmac{i}{u}{k}{h}(x) - \weightmac{i}{u}{k}{h}(x\flipBit{q}) \right| = & \left|\frac{c_u}{\sum_{y=1}^\seqlen c_y} - \frac{c_u+d_u}{\sum_{y=1}^\seqlen c_y+d_y} \right| \\
    = & \left|\frac{c_u(\sum_{y=1}^\seqlen c_y+d_y)-(c_u+d_u)\sum_{y=1}^\seqlen c_y}{(\sum_{y=1}^\seqlen c_y) (\sum_{y=1}^\seqlen c_y+d_y)} \right| \\
    = & \left| \frac{c_u\sum_{y=1}^\seqlen d_y - d_u \sum_{y=1}^\seqlen c_y}{(\sum_{y=1}^\seqlen c_y) (\sum_{y=1}^\seqlen c_y+d_y)} \right|\\
    \leq & \frac{c_u \sum_{y=1}^\seqlen \influence_q[c_y] + \influence_q[c_u] \sum_{y=1}^\seqlen c_y}{\seqlen^2 \exp(-2C_A^{(k)})}  
    \end{align*}
    Now using
    \begin{equation}\label{eq:influence-exp-logit-bound}
    \begin{aligned}
        \influence_q[c_u] =& \influence_q\left[\exp\left(a_{iu}^{(k,h)}(x)\right)\right] \\
        \leq & \exp(C_A^{(k)}) \influence_q\left[a_{iu}^{(k,h)}(x)\right] \\
	    \leq & \exp(C_A^{(k)}) \|K_{k,h}^T Q_{k,h}\|  
	    \left( 
	    \influence_q\left[\ypost_{i}^{(k-1)}(x)\right] 
	    + \influence_q\left[\ypost_{u}^{(k-1)}(x)\right] 
	    \right) \\
    \end{aligned}
    \end{equation}
    By $x \leq \exp(x)$, we can bound $\|K_{k,h}^T Q_{k,h}\|_2 \leq \exp(C_A^{(k)})$.
Now taking a sum over $u$ yields:
\begin{align*}
	\sum_u  \left|\weightmac{i}{u}{k}{h}(x) - \weightmac{i}{u}{k}{h}(x\flipBit{q}) \right| \leq  & \sum_{u=1}^{\seqlen}\frac{c_u \sum_{y=1}^{\seqlen} \influence_q[c_y] + \influence_q[c_u] \sum_{y=1}^{\seqlen}c_y}{\seqlen^2 \exp(-2C_A^{(k)})}  \\
  \leq^{(\ref{eq:influence-exp-logit-bound}, \ref{eq:cu-bound})}  & \exp(2C_A^{(k)})   \frac{2\seqlen^2\influence_q\left[\ypost_{i}^{(k-1)}(x)\right] + 
	    2 \seqlen\sum_{u=1}^{\seqlen}\influence_q\left[\ypost_{u}^{(k-1)}(x)\right] 
	     }{\seqlen^2 \exp(-2C_A^{(k-1)})}  \\
    \leq & 2  \exp(4C_A^{(k-1)})  \left( \influence_q[\ypost_i^{(k-1)}] + \frac{1}{\seqlen} \sum_{j=1}^{\seqlen}  \influence_q[\ypost_j^{(k-1)}] \right)
\end{align*}
\end{proof}

\paragraph{Bounding the Sensitivity of the Attention Heads}
We now put together Lemmas~\ref{lemma:head-influence-bound},\ref{lemma:influence-att-higher} in order to bound the sensitivity of any individual attention head's output.
\begin{lemma}
    For $k>1$,  the following holds for each $j$ and $h$:
    \begin{equation}\label{eq:inf-head-large-k}
	    \influence_q[{\headmac{j}{h}{k}}] \leq  4 \sqrt{d}   \exp(4C_A^{(k)}) \left[ \influence_q[\ypost_j^{(k-1)}] + \frac{ 1}{\seqlen} \sum_{w=1}^\seqlen \influence_q[y_w^{(k-1)}]   
        \right]   
\end{equation}
\end{lemma}
\begin{proof}
%
%
%
In the higher layers ($k>1$):
\begin{align*}
\influence_q[\headmac{j}{h}{k}] 
\leq^{(\ref{eq:lemma:head-influence-bound})} 
&  \sum_{w=1}^{\seqlen}\hat{a}_{j,w}^{k,h} \influence_q[y_w^{(k-1)}] +  \left(\max_w \|y_w^{(k-1)}\|\right)  \sum_{w=1}^{\seqlen} \influence_q[\hat{a}_{j,w}^{k,h}]  
\\
\leq^{(\ref{eq:higher-layers-att-influence})} &   \frac{\exp(2C_A^{(k)})}{\seqlen} \sum_{w=1}^\seqlen \influence_q[y_w^{(k-1)}] +  \left(\max_w \left\|y_w^{(k-1)}\right\|\right) 2  \exp(4C_A^{(k)})  
     \left( \influence_q[\ypost_j^{(k-1)}] + \frac{1}{\seqlen} \sum_{j=1}^{\seqlen}  \influence_q[\ypost_j^{(k-1)}] \right)\\ 
= &   \frac{\exp(2C_A^{(k)}) + \sqrt{d} 2  \exp(4C_A^{(k)}) }{\seqlen} \sum_{w=1}^\seqlen \influence_q[y_w^{(k-1)}] +  \sqrt{d} 2  \exp(4C_A^{(k)})  
     \left( \influence_q[\ypost_j^{(k-1)}]  \right)\\     
    \leq &  4 \sqrt{d}   \exp(4C_A^{(k)}) \left[ \influence_q[\ypost_j^{(k-1)}] + \frac{ 1}{\seqlen} \sum_{w=1}^\seqlen \influence_q[y_w^{(k-1)}]   
        \right] \\   
\end{align*}
where the second step uses the fact that \begin{equation}
    \hat{a}_{j,w}^{k,h} \leq^{(\ref{eq:cu-bound})} \frac{\exp(2C_A^{(k)})}{\seqlen},
\end{equation} and the third step uses the fact that the outcome of layer norm has an L2 norm bounded by $\sqrt{d}$. 
\end{proof}

\subsection{Impact of Layer Norm}
Next, we investigate the effect that layer norm has on sensitivity. We note that layer norm was not fully included in any of the pointwise Lipschitzness bounds \citep{hahn2020theoretical, edelman2022inductive, li2023stability}: \citet{hahn2020theoretical,  li2023stability} do not seem to mention layer norm, and \citet{edelman2022inductive} assumes a Lipschitz-continuous proxy, projection on the unit ball.
We find that standard layer norm indeed can increase sensitivity substantially, which will be key to our overall results:
\begin{lemma}\label{lemma:influence-layer-norm}
Let $\ypre(x) \in \mathbb{R}^d$, and $\ypost(x) := LayerNorm(\ypre(x))$.
Define $N(x) := \sqrt{\sigma^2(\ypre(x)) + \epsilon}$.
Let $C>1$ be such that $\|\ypre(x)\|_2 \leq C$.
Then
    \begin{equation}\label{eq:layer-norm-influence-bound}
    \influence_i[\ypost] \leq 2 \cdot \influence_i[\ypre] \cdot C \cdot \left[ 1 + \frac{ 1
    }{\blowupInLemma(x)}  \right] \cdot \left[ 1 + \frac{ 1
    }{\blowupInLemma(x\flipBit{i})}  \right] 
\end{equation}
\end{lemma}
\begin{proof}
First, we note
\begin{align*}
    \influence_i[(\ypre-\mu(\ypre))] \leq \influence_i[\ypre]+\frac{1}{d}\sum_s \influence_i[\ypre_s] = (1+\frac{1}{d}) \influence_i[\ypre] \leq 2 \influence_i[\ypre]
\end{align*}
Thus, at a factor of $2$, it is sufficient to consider the case where $\mu(\ypre) = 0$, because $\|\ypre-\mu(\ypre)\| \leq \|\ypre\|$.
Then
\begin{align*}
    \influence_i[\ypost] =& \|\ypost(x) - \ypost(x\flipBit{i})\|_2 \\
    =& \left\|\frac{\ypre(x)}{\blowupInLemma(x)} - \frac{\ypre(x\flipBit{i})}{\blowupInLemma(x\flipBit{i})}\right\|_2 \\
    =& \frac{\left\|\blowupInLemma(x\flipBit{i}) \ypre(x)  - \blowupInLemma(x)\ypre(x\flipBit{i})\right\|_2}{\blowupInLemma(x)\blowupInLemma(x\flipBit{i})} \\
     \leq & \frac{
\|\ypre(x)\|_2 \cdot  |\blowupInLemma(x\flipBit{i})-\blowupInLemma(x)| + \blowupInLemma(x\flipBit{i}) \left\|\ypre(x)-\ypre(x\flipBit{i})\right\|_2
    }{\blowupInLemma(x)\blowupInLemma(x\flipBit{i})} \\
    \leq & \frac{
\|\ypre(x)\|_2 \cdot |\blowupInLemma(x\flipBit{i})-\blowupInLemma(x)|}{\blowupInLemma(x)\blowupInLemma(x\flipBit{i})} + \frac{ \left\|\ypre(x)-\ypre(x\flipBit{i})\right\|_2
    }{\blowupInLemma(x)} \\    
    \leq & \frac{
C |\blowupInLemma(x\flipBit{i})-\blowupInLemma(x)|}{\blowupInLemma(x)\blowupInLemma(x\flipBit{i})} + \frac{ \left\|\ypre(x)-\ypre(x\flipBit{i})\right\|_2
    }{\blowupInLemma(x)} \\    
    \leq &\influence_i[\ypre] \cdot \left[ \frac{
C}{\blowupInLemma(x)\blowupInLemma(x\flipBit{i})} + \frac{ 1
    }{\blowupInLemma(x)}  \right]\\ 
    = &\influence_i[\ypre] \cdot \frac{1}{\blowupInLemma(x)} \cdot \left[ \frac{
C}{\blowupInLemma(x\flipBit{i})} + 1  \right]\\ 
    \leq &\influence_i[\ypre] \cdot C \cdot \left[ 1 + \frac{ 1
    }{\blowupInLemma(x)}  \right] \cdot \left[ 1 + \frac{ 1
    }{\blowupInLemma(x\flipBit{i})}  \right]\\ 
\end{align*}
where $C\geq 1$ is a bound on $\|\ypre_i\|_2$.
Here, we used $ \influence_i[\blowupInLemma] \leq  \influence_i[\ypre]$ in the last step:
Note that $\ypre$ is mean zero; hence, standard deviation equals the L2 norm, and
\begin{align*}
  \sqrt{d} \influence_i[\blowupInLemma] =  | \|\ypre(x)\|_2-\|\ypre(x\flipBit{i})\|_2| \leq \|\ypre(x)-\ypre(x\flipBit{i})\|_2 \leq \|\ypre(x)-\ypre(x\flipBit{i})\|_2 = \influence_i[\ypre]
\end{align*}

\end{proof}

\subsection{Layerwise Bounds on Sensitivity}
Putting together Lemmas~\ref{lemma:head-influence-bound} and \ref{lemma:influence-layer-norm}, we bound the influence of the $i$-th bit on the $k$-th layer in terms of its influence on the $k-1$-th layer.

\begin{lemma}
With probability $\geq 1-\frac{H}{\seqlen^{2}}$ over the choice of $x$, the following holds for each $j$:
\begin{equation} \label{eq:prenorm-layer1-bound}
\boxed{
  \sum_{i=1}^{\seqlen}  \influence_i[\ypre_j^{(1)}] \leq  
150 L_{f^{MLP}} C_V^2 C_P H \log{\seqlen} = C^{(0)} \log{\seqlen}
}
\end{equation}
and
\begin{equation}
\boxed{
    \influence_i[\ypost_j^{(1)}] \leq 2 \left( \max_w \|\ypre_w^{(1)}\|_{2}\right) \cdot \normFactor_j^{(1)}(x) \cdot \normFactor_j^{(1)}(x\flipBit{i}) \cdot L_{f^{MLP}} \cdot \influence_i[\ypre_j^{(1)}]
    }
\end{equation}
For $k>1$:
\begin{equation}\label{eq:lemma-influence-recursion-kLarge}
\boxed{
     \influence_i[\ypost_j^{(k)}] \leq    \frac{C^{(k)}}{3} \cdot\normFactor_j^{(k)}(x) \cdot \normFactor_j^{(k)}(x\flipBit{i})  \cdot \left(  
       \sum_{w=1}^{\seqlen} (\delta_{w,j} + \frac{1}{\seqlen}) \influence_i[y_w^{(k-1)}] \right)
}
\end{equation}
\end{lemma}
Intuitively, the $\delta_{w,j}$ contribution reflects the skip connection (residual stream); the $\frac{1}{\seqlen}$ contribution reflects the role of soft attention.
\begin{proof}
We first obtain:
\begin{equation}\label{eq:lemma-mlp}
 \influence_s[\ypost_j^{(k)}] \leq^{(\ref{eq:layer-norm-influence-bound},\ref{eq:def:ypre},\ref{eq:layer-norm-and-mlp}, \ref{app:eq:cmlp})} \normFactor_k^{(k)}(x) \cdot \normFactor_k^{(k)}(x\flipBit{s}) \cdot C_{MLP}  \cdot L_{f^{MLP}} \cdot \left( \influence_s[y_j^{(k-1)}] +  \sum_{q=1}^H \influence_s[\headmac{i}{q}{k}] \right) 
\end{equation}

For $k>1$, we find, putting together the previous results:
\begin{align*}
     \influence_i[\ypost_j^{(k)}] \leq^{(\ref{eq:layer-norm-influence-bound})} &  \influence_i[\ypre_j] \cdot \normFactor_j^{(k)}(x) \cdot \normFactor_j^{(k)}(x\flipBit{i}) \\
	\leq^{(\ref{eq:lemma-mlp})} &  (L_{f^{MLP}}  C_{MLP} \cdot \left(\max_w \|\ypre_w^{(k)}\|_{2}\right) \cdot \left( \influence_i[y_j^{(k-1)}] +  \sum_{q=1}^H \influence_i[\headmac{i}{q}{k}] \right)) \cdot \normFactor_j^{(k)}(x) \cdot \normFactor_j^{(k)}(x\flipBit{i}) \\
    \leq^{(\ref{eq:inf-head-large-k})} &  (L_{f^{MLP}}  C_{MLP} \cdot C_{MLP} (H+1) C_V C_P \cdot \left( \influence_i[y_j^{(k-1)}] +  H
 4 \sqrt{d}   \exp(4C_A^{(k)}) \left[ \influence_i[\ypost_j^{(k-1)}] + \frac{ 1}{\seqlen} \sum_{w=1}^\seqlen \influence_i[y_w^{(k-1)}]   
        \right]   
    \right)) \cdot \normFactor_j^{(k)}(x) \cdot \normFactor_j^{(k)}(x\flipBit{i}) \\
    =^{(\ref{eq:def-ck})} &  \frac{C^{(k)}}{3} \cdot \normFactor_j^{(k)}(x) \cdot \normFactor_j^{(k)}(x\flipBit{i}) \cdot \left(  
       \sum_{w=1}^{\seqlen} (\delta_{w,j} + \frac{1}{\seqlen}) \influence_i[y_w^{(k-1)}] \right)
\end{align*}
In the lowest layer, the influence is bounded instead in terms of the influence on the attention weights. With high probability, a logarithmic bound on the summed influences over all bits results.
Crucially, the key-value-matrix plays no role here.
For $k=1$ we find, setting $\delta=4$ in (\ref{eq:lemma:attention-bound-first-layer}):
\begin{align*} 
     \influence_i[\ypre_j^{(1)}] \leq^{(\ref{eq:def:ypre})} & \influence_i\left[f_{MLP}\left(\ypost_j^{(0)} + \sum_{h=1}^H \headmac{j}{h}{1}\right)\right]\\
     \leq^{(\ref{app:eq:lfmlp})} & L_{f^{MLP}} \left[\influence_i\left[\ypost_j^{(0)}\right] + \sum_{h=1}^H \influence_i\left[\headmac{j}{h}{1}\right]\right]\\
     \leq^{(\ref{eq:lemma:head-influence-bound})} & L_{f^{MLP}} \left[\influence_i\left[\ypost_j^{(0)}\right] + \sum_{h=1}^H C_V \left[\sum_{w=1}^{\seqlen}\hat{a}_{j,w}^{1,h} \influence_i[y_w^{(0)}] + C_V \left(\max_w \|y_w^{(0)}\|_2\right)  \sum_{w=1}^{\seqlen} \influence_i[\hat{a}_{j,w}^{1,h}]\right]\right]\\
     \leq^{(\ref{app:eq:cp})} & L_{f^{MLP}} \left[2C_P \delta_{ij} + C_V \sum_{h=1}^H  \left[\sum_{w=1}^{\seqlen}\hat{a}_{j,w}^{1,h} C_P \delta_{iw} + C_V C_P  \sum_{w=1}^{\seqlen} \influence_i[\hat{a}_{j,w}^{1,h}]\right]\right]\\
     \leq & L_{f^{MLP}} C_V^2 C_P \left[2 \delta_{ij} +  \sum_{h=1}^H  \hat{a}_{j,i}^{1,h}  +   \sum_{h=1}^H \sum_{w=1}^{\seqlen} \influence_i[\hat{a}_{j,w}^{1,h}]\right]
\end{align*}
and hence, by summing over $i$,
\begin{align*} 
  \sum_{i=1}^{\seqlen}  \influence_i[\ypre_j^{(1)}] \leq &  L_{f^{MLP}} C_V^2 C_P \left[2 + \sum_{h=1}^H \sum_{i=1}^\seqlen   \hat{a}_{j,i}^{1,h}  + \sum_{h=1}^H  \sum_{i=1}^\seqlen \sum_{w=1}^{\seqlen} \influence_i[\hat{a}_{j,w}^{1,h}]\right] \\
  \leq^{(\ref{eq:lemma:attention-bound-first-layer})} &  L_{f^{MLP}} C_V^2 C_P \left[2 + H  + H  136\log{\seqlen} + 10H\right] \\
  \leq & 150 L_{f^{MLP}} C_V^2 C_P H \log{\seqlen}  
\end{align*}
with probability $\geq 1-\frac{H}{\seqlen^{2}}$.\footnote{The last inequality holds whenever $\seqlen > 1$. If $\seqlen=1$, the probability bound evaluates to 0, making the statement trivially true.}

\end{proof}

\subsection{Relating Influence at First and Last Layers}
At this point, we have derived bounds on the influence of any bit on any individual activation $\ypost_j^{(k)}$ in terms of the parameter norms, the layer-norm-induced blowup, and influence the bit has on the activations at the preceding layer.
We now derive bounds linking these quantities to the overall sensitivity $s(f,x)$ and the average sensitivity $as_\seqlen(f)$.
In order to do this, we link influences on the top layer to influences at the first layer.
Recall
\begin{equation}
    \Blowup_K(x) = \prod_{k=1}^K \tau^{(k)}(x)
\end{equation}
We first find that
\begin{lemma}
\begin{equation}\label{eq:ineq:inf-b}
    \influence_i[\ypost_j^{(k)}] 
    \leq  C^{(1)} \left(\prod_{l=2}^k C^{(k)}\right)   \Blowup_k(x) \Blowup_k(x\flipBit{i})  
        \left(
       \frac{1}{\seqlen} \sum_{w=q}^\seqlen \influence_i[\ypre_w^{(1)}] + \influence_i[\ypre_j^{(1)}]
        \right)
\end{equation}
\end{lemma}

\begin{proof}
By induction using~(\ref{eq:lemma-influence-recursion-kLarge}).
First, for $k=2$, 
\begin{align*}
    \influence_i[\ypost_j^{(2)}] 
    \leq &  C^{(2)} \cdot\normFactor_j^{(2)}(x) \cdot \normFactor_j^{(2)}(x\flipBit{i})  \cdot \left(  
       \sum_{w=1}^{\seqlen} (\delta_{w,j} + \frac{1}{\seqlen}) \influence_i[y_w^{(1)}] \right) \\
    \leq   &  C^{(2)} \cdot\normFactor_j^{(2)}(x) \cdot \normFactor_j^{(2)}(x\flipBit{i})  \cdot \left(  
       \sum_{w=1}^{\seqlen} (\delta_{w,j} + \frac{1}{\seqlen}) 2 \left( \max_w \|\ypre_w^{(1)}\|_{2}\right) \cdot \normFactor_w^{(1)}(x) \cdot \normFactor_w^{(1)}(x\flipBit{i}) \cdot L_{f^{MLP}} \cdot \influence[\ypre_w^{(1)}] \right) \\
    \leq   &  C^{(2)} 2 \left( \max_w \|\ypre_w^{(1)}\|_{2}\right)   \cdot L_{f^{MLP}} \cdot \Blowup_{2}(x) \cdot \Blowup_2(x\flipBit{i}) \cdot \left(  
       \sum_{w=1}^{\seqlen} (\delta_{w,j} + \frac{1}{\seqlen})  \cdot \influence[\ypre_w^{(1)}] \right) \\
           \leq   &  C^{(1)} C^{(2)} \cdot \Blowup_{2}(x) \cdot \Blowup_2(x\flipBit{i}) \cdot \left(  
       \sum_{w=1}^{\seqlen} (\delta_{w,j} + \frac{1}{\seqlen})  \cdot \influence[\ypre_w^{(1)}] \right)
\end{align*}
where
\begin{align*}
    C^{(1)} = 2 \left( \max_w \|\ypre_w^{(1)}\|_{2}\right)   \cdot L_{f^{MLP}}
\end{align*}
Now, for $k>2$, using the induction hypothesis (IH):
\begin{align*}
     \influence_i[\ypost_j^{(k)}] \leq &   \frac{C^{(k)} }{3} \cdot\normFactor_j^{(k)}(x) \cdot \normFactor_j^{(k)}(x\flipBit{i})  \cdot \left(  
       \sum_{w=1}^{\seqlen} (\delta_{w,j} + \frac{1}{\seqlen}) \influence_i[y_w^{(k-1)}] \right) \\
     \leq^{(IH)}   &   \frac{1}{3} C^{(1)}\left(\prod_{l=2}^k C^{(k)}\right)   \Blowup_k(x) \Blowup_k(x\flipBit{i})    \cdot \left(  
       \sum_{w=1}^{\seqlen} 
       \sum_{v=1}^\seqlen
       (\delta_{w,j} + \frac{1}{\seqlen})
       (\delta_{v,w} + \frac{1}{\seqlen})
          \influence_i[\ypre_v^{(1)}]
        \right) \\
         =  & \frac{1}{3} C^{(1)}\left(\prod_{l=2}^k C^{(k)}\right)   \Blowup_k(x) \Blowup_k(x\flipBit{i})    \cdot \left(  
       \influence_i[\ypre_j^{(1)}] + 
       3 \sum_{v=1}^\seqlen\frac{1}{\seqlen} \influence_i[\ypre_v^{(1)}]
        \right) \\
         \leq  & C^{(1)}\left(\prod_{l=2}^k C^{(k)}\right)   \Blowup_k(x) \Blowup_k(x\flipBit{i})    \cdot \left(  
       \influence_i[\ypre_j^{(1)}] + 
       \sum_{v=1}^\seqlen\frac{1}{\seqlen} \influence_i[\ypre_v^{(1)}]
        \right)
\end{align*}

\end{proof}

\subsection{Deriving Almost-Everywhere Pointwise Sensitivity Bounds}\label{sec:app:pointwise}

Putting together the previous findings, we are now in a position to link the sensitivity to layer norm blowup and parameter norms.
We find that the sensitivity on an input $x$ is bounded in terms of the blowup on $x$ itself and on its Hamming neighbors, up to sublinear factors $\sqrt{n\log{\seqlen}}$:
\begin{thm}[Repeated from Theorem~\ref{thm:sensitivity-blowup-bound}]
With probability at least $1-\frac{H}{\seqlen^{-2}}$ over the choice of $x \in \{\pm 1\}^\seqlen$, we have
\begin{align*}
\frac{s(f,x)}{C \sqrt{\seqlen \log\seqlen}} 
\leq  \Blowup(x)^2 +\frac{1}{\seqlen} \sum_i  \Blowup(x\flipBit{i})^2 
\end{align*}

\end{thm}

\begin{proof}

The most immediate idea is to sum both sides of (\ref{eq:ineq:inf-b}) across $i$ to get a bound on $s(f,x) = \frac{1}{4} \sum_{i=1}^\seqlen \influence_i[\ypost_\seqlen^{(\numLayers)}]^2$. 
In fact, it will pay off to \emph{average} the influence across the input bits, and only later eliminate the $1/\seqlen$ factor.
Formally,
\begin{align*}
\frac{1}{\seqlen} s(f,x) =^{(\ref{eq:sens-infl-link})} & \frac{1}{4\seqlen} \sum_{i=1}^\seqlen \influence_i[v^T \ypost_{\seqlen}^{(\numLayers)}]^2 \\
\leq^{(\ref{eq:def:influence})} & \frac{1}{4\seqlen} \sum_{i=1}^\seqlen \influence_i[ \ypost_{\seqlen}^{(\numLayers)}]^2 \|v_{out}\|_2^2  \\
\leq^{(\ref{eq:l1-inf-bounds-l2-inf})} & \frac{1}{2\seqlen} \sum_{i=1}^\seqlen \influence_i[ \ypost_{\seqlen}^{(\numLayers)}] 
 \|v_{out}\|_2^2 \sqrt{d} \\
    \leq^{(\ref{eq:ineq:inf-b})} & 
   \frac{\sqrt{d} \|v_{out}\|^2}{2} \left(\prod_{l=2}^k C_k\right)  \Blowup(x) \sum_i  \frac{1}{\seqlen}   \Blowup(x\flipBit{i}) \left[\frac{1}{\seqlen} \sum_w \influence_i[\ypre_w^{(1)}] +\influence_i[\ypre_\seqlen^{(1)}] \right] \\
    \leq^{(CSB)} & \frac{\sqrt{d} \|v_{out}\|^2}{2}\left(\prod_{l=2}^k C_k\right)  \Blowup(x)   \sqrt{\sum_i  \frac{1}{\seqlen}\Blowup(x\flipBit{i})^2} \left[ \sqrt{\sum_i  \frac{1}{\seqlen} \frac{1}{\seqlen} \sum_w \influence_i[\ypre_w^{(1)}]^2} + \sqrt{\sum_i  \frac{1}{\seqlen}\influence_i[\ypre_\seqlen^{(1)}]^2}\right] 
    \end{align*}
    where the last inequality uses Cauchy-Schwarz-Bunyakovsky (CSB).
Now, we find for any $w \in \{1,\dots,\seqlen\}$:
\begin{equation}\label{eq:layer-1-summed-influence-final}
  \sum_i  \influence_i[\ypre_w^{(1)}]^2 \leq^{(\ref{eq:l1-inf-bounds-l2-inf})} 2 \left(\max_x \|\ypre_w^{(1)}\|_2\right) \sum_i \influence_i[\ypre_w^{(1)}] \leq^{(\ref{eq:prenorm-layer1-bound})} \left(\max_x \|\ypre_w^{(1)}\|_2\right) C_0  \log{\seqlen}
\end{equation}
with probability (simultaneously over $w$) $1-\frac{H}{\seqlen^{\delta-2}}$ over the choice of $x$.
Setting $\delta =4$ and defining
    and plugging in (\ref{eq:layer-1-summed-influence-final}) together with the fact $\sqrt{x}\leq x$ when $x\geq 1$, we get:
    \begin{align*}
  \frac{1}{\seqlen} s(f,x) 
\leq & C \Blowup(x) \sqrt{\frac{ \log{\seqlen}}{n}}  \sqrt{\sum_i  \frac{1}{\seqlen}\Blowup(x\flipBit{i})^2}  
\end{align*}
Rearranging, we find
\begin{align*}
  \frac{1}{C} s(f,x) \sqrt{\frac{1}{ \seqlen \log{\seqlen}}} \leq &   \Blowup(x)   \sqrt{\sum_i  \frac{1}{\seqlen}\Blowup(x\flipBit{i})^2}  \\
\leq & \Blowup(x)^2 + \sum_i  \frac{1}{\seqlen}\Blowup(x\flipBit{i})^2  
\end{align*}
where the last step follows by Young's Inequality ($ab \leq \frac{a^2+b^2}{2}$ for $a,b\geq 0$).
As we chose $\delta=4$, this is true with probability $1-\frac{H}{n^2}$ over the choice of $x$.
\end{proof}

\subsection{Deriving On-Average Sensitivity Bounds}\label{sec:app:average}
Finally, we can convert the high-probability statement from the preceding lemma, which bounded the layer norm blowup on $x$ itself \emph{or} its Hamming neighbors, into an on-average bound over the entire input space:
\begin{corollary}[Repeated from Corollary \ref{thm:sensitivity-blowup-bound}]
\begin{equation}
 C \cdot  \mathbb{E}[\Blowup(x)^2] \geq       \frac{as_\seqlen(f)}{\sqrt{\seqlen\log{\seqlen}}}  -   \frac{H}{n}
\end{equation}

\end{corollary}
\begin{proof}
Recall from the preceding lemma that, with probability $1-\frac{H}{\seqlen^{2}}$ over the choice of $x$, we have
\begin{equation}\label{eq:from-preceding-lemma}
    \frac{s(f,x)}{C\sqrt{\seqlen\log\seqlen}}  
    \leq  \Blowup(x)^2 +\frac{1}{\seqlen} \sum_i  \Blowup(x\flipBit{i})^2 
\end{equation}
We first note the following, for any function $\phi : \{\pm 1\}^{\seqlen}\rightarrow \mathbb{R}$:
\begin{align*}
    \mathbb{E}[\phi(x)] = \frac{1}{\seqlen} \mathbb{E}_{x} \sum_{i=1}^{\seqlen}\phi(x\flipBit{i})
\end{align*}
since every one of the maps $x \mapsto x\flipBit{i}$ ($i=1,\dots, n$) is a bijection on $\{\pm 1\}$.
Hence,
\begin{align*}
    \mathbb{E}[\phi(x)] = \mathbb{E}_{x} \left[\frac{1}{2} \phi(x) + \frac{1}{2}\sum_{i=1}^{\seqlen}\phi(x\flipBit{i})\right]
\end{align*}
Taking $\phi(x) := \Blowup(x)^2$ yields:
\begin{equation}
\mathbb{E}[\Blowup(x)^2] = \mathbb{E}_{x}  \left[\frac{1}{2} \Blowup(x)^2 + \frac{1}{2} \frac{1}{\seqlen} \sum_i \Blowup(x\flipBit{i})^2\right]
\end{equation}
    Let $A$ be the set of $x$ satisfying (\ref{eq:from-preceding-lemma}); $|A| \geq 2^{\seqlen}(1-\frac{H}{\seqlen^{2}})$. Then
\begin{align*}
\mathbb{E}[\Blowup(x)^2] = & \frac{1}{2^n} \sum_{x \in \{\pm 1\}^\seqlen}  \left[\frac{1}{2} \Blowup(x)^2 + \frac{1}{2} \frac{1}{\seqlen} \sum_i \Blowup(x\flipBit{i})^2\right] \\
\geq &  \frac{1}{2^n} \sum_{x\in A}  \left[\frac{1}{2} \Blowup(x)^2 + \frac{1}{2} \frac{1}{\seqlen} \sum_i \Blowup(x\flipBit{i})^2\right] \\
\geq &   \frac{1}{C\sqrt{\seqlen\log\seqlen}} \frac{1}{2^\seqlen}  \sum_{x\in A}  s(f,x) \\
\geq &  \frac{1}{C\sqrt{\seqlen\log\seqlen}} \frac{1}{2^\seqlen} \sum_{x \in \{\pm 1\}^\seqlen}   s(f,x) - \frac{1}{C\sqrt{\seqlen\log\seqlen}} \frac{1}{2^\seqlen} \sum_{x\in\in A}   s(f,x) \\
\geq &  \frac{1}{C\sqrt{\seqlen\log\seqlen}}  as_n(f) - \frac{1}{C\sqrt{\log\seqlen}} \frac{H}{\seqlen^{3/2}}   \\
\end{align*}
Rearranging, we find:
\begin{align*}
 C \cdot  \mathbb{E}[\Blowup(x)^2] \geq &  \frac{1}{\sqrt{\seqlen\log\seqlen}}  as_n(f) - \frac{1}{\sqrt{\log\seqlen}} \frac{H}{\seqlen^{3/2}}  \geq   \frac{1}{\sqrt{\seqlen\log\seqlen}}  as_n(f) -  \frac{H}{\seqlen}  \\
\end{align*}

\end{proof}

\section{Theorem~\ref{thm:lrho-bound}}\label{appendix:sec:theorem-sharpness}

Before proving Theorem~\ref{thm:lrho-bound}, we require the following lemma in the analysis of Boolean functions. Informally, the lemma describes how the RMSE between two functions can be lower-bounded in terms of their average sensitivities.

For any function $f : \{\pm 1\}\rightarrow \mathbb{R}$, we write $\|f\|_2 := \sqrt{\frac{1}{2^\seqlen}\sum_{x\in\{\pm 1\}^\seqlen} f(x)^2} = \mathbb{E}_x[f(x)^2]$. Then the lemma states:

\begin{lemma}\label{lemma:as-rmse-diff}

For any $f, f' : \{\pm 1\}^\seqlen\rightarrow\mathbb{R}$, we have 
for any $\alpha > 1$:
\begin{equation}
\sqrt{\mathbb{E}[(f(x)-f'(x))^2]} \geq \sqrt{\frac{(1-\frac{1}{\alpha}) as_\seqlen(f)}{(\seqlen-\frac{as_\seqlen(f)}{\alpha \cdot \|f\|_2^2})}} - \|f\|_2 \sqrt{\alpha \frac{as_\seqlen(f')}{{as_\seqlen(f)}}}
\end{equation}
Note that the optimal $\alpha$ depends on $f, f'$. 
The choice $\alpha = 2$ will be sufficient for our needs:
\begin{equation}
\sqrt{\mathbb{E}[(f(x)-f'(x))^2]} \geq \sqrt{\frac{as_\seqlen(f)}{2\seqlen-\frac{as_\seqlen(f)}{ \|f\|_2^2}}} - \|f\|_2 \sqrt{2 \frac{as_\seqlen(f')}{{as_\seqlen(f)}}}
\end{equation}

\end{lemma}

\begin{proof}
Recall the discussion of the Fourier-Walsh decomposition preceding (\ref{eq:sensitivity-fourier}).
The proof idea is that, when functions have different average sensitivities, their degree profiles must be quite different, which then entails that they must be substantially distinct in behavior.

     Let $1 \leq \lambda_{Min} \leq \seqlen$; we will fix it later.
Let $\Pi$ be the orthogonal projection on the Fourier components with degree at least $\lambda_{Min}$, formally given as the unique linear map satisfying
\begin{equation}
\Pi \chi_P = \begin{cases}\chi_P & \text{if }|P| \geq \lambda_{Min}\\ 0 & \text{else}\end{cases}
\end{equation}
    Then, for any $g : \{\pm 1\}^\seqlen\rightarrow\mathbb{R}$, by (\ref{eq:sensitivity-fourier}) and Parseval's theorem for the Fourier-Walsh transform \citep{odonnell2014analysis}:\footnote{If $d_0 \dots, d_\seqlen$ is the degree profile of $g$, then $\lambda_{Min}  \|\Pi g\|_2^2 = \sum_{i \geq \lambda_{Min}} \lambda_{Min} d_i \leq as_\seqlen(g) \leq \sum_{0 < i < \lambda_{Min}} \lambda_{Min} d_i + \sum_{i \geq \lambda_{Min}} \seqlen d_i = \lambda_{Min} (\|g\|_2^2-\|\Pi g\|_2^2) + \seqlen \|\Pi g\|_2^2 $, where the inequalities follow from (\ref{eq:sensitivity-fourier}) and the equalities follow from Parseval's theorem applied to $\Pi g$ and to $g$.}
    \begin{equation}
     \lambda_{Min}  \|\Pi g\|_2^2 \leq   as_\seqlen(g)  \leq \lambda_{Min} (\|g\|_2^2-\|\Pi g\|_2^2) + \seqlen \|\Pi g\|_2^2
    \end{equation}
    Rearranging and inserting $f, f'$ for $g$ in the two inequalities, respectively:
    \begin{align*}
   \frac{   as_\seqlen(f) - \lambda_{Min} \|f\|_2^2 }{\seqlen-\lambda_{Min} } \leq  & \|\Pi f\|_2^2 \\
          \|\Pi f'\|_2^2 \leq  & \frac{as_\seqlen(f)}{\lambda_{Min}}
    \end{align*}
    Hence
    \begin{align*}
        \|\Pi f\|_2 - \|\Pi f'\|_2 \geq \sqrt{\frac{   as_\seqlen(f) - \lambda_{Min}\|f\|_2^2}{(\seqlen-\lambda_{Min})}} - \sqrt{\frac{as_\seqlen(f')}{\lambda_{Min}}}
    \end{align*}
Now at $\lambda_{Min} = \frac{as_\seqlen(f)}{\alpha \|f\|_2^2}$, we get
\begin{align*}
        \|\Pi f\|_2 - \|\Pi f'\|_2 \geq \sqrt{\frac{(1-\frac{1}{\alpha}) as_\seqlen(f)}{(\seqlen-\frac{as_\seqlen(f)}{\alpha \|f\|_2^2})}} - \|f\|_2 \sqrt{\alpha \frac{as_\seqlen(f')}{{as_\seqlen(f)}}}
    \end{align*}
Now by the reverse triangle inequality
\begin{equation}
    \|\Pi f\|_2-\|\Pi f'\|_2 \leq \||\Pi f - \Pi f'\|_2 \leq \| f -  f'\|_2
\end{equation}
and the claim follows.

\end{proof}

\begin{proof}[Proof of the Theorem]
The proof can be decomposed into the following steps.

\paragraph{Step 1: Decomposing $\Delta$}
We have $(F \cdot d + G \cdot d^2) \cdot L$ parameters (where $F, G$ are constant).
We are assuming that $\Delta$ is from $\rho \mathbb{S}$. 
For each of the $\numLayers$ instances layer norm, we consider the part of $\Delta$ corresponding to the immediately preceding bias; we will name this $\Delta_1, \dots, \Delta_{\numLayers}$.

\paragraph{Step 2: Effect of Perturbations on $C_\theta$}
We will make explicit the dependence of $C$ on $\theta$ by writing $C_\theta$.
We next note that, for any fixed $\theta$, as $\rho \rightarrow 0$,  the effect of a small perturbation $\Delta$ on $C_\theta$ is bounded, because $C_\theta$ is locally Lipschitz in $\theta$.\footnote{This follows from the fact that in (\ref{eq:def:c}), for $C_P$, $\ell^1, \ell^\infty$ norms are locally Lipschitz; for $C_{MLP}$ because it can be chosen to be the product of parameter matrix norms in the MLP; for the exponential term because the spectral norm is locally Lipschitz.}
Thus, as $\rho\rightarrow 0$:
\begin{equation}\label{eq:scaling-c-rho}
    C_{\theta+\Delta} = C_\theta \cdot (1+O(\rho))
\end{equation}
uniformly across all $\Delta$ with $\|\Delta\|=\rho$,
where $\mathcal{O}(\cdot)$ includes constants depending on $\theta$, but not $\rho$, $\Delta$, or $n$.

\paragraph{Step 2: Lower Bounding Error in Terms of Sensitivity Drop}
Set
\begin{equation}
    \mu := \liminf_{n\rightarrow \infty} \frac{as_\seqlen(T_\theta)}{n}
\end{equation}
Necessarily, $\mu \in [0,1]$.
The claim of the theorem is nontrivial if and only if $\mu > 0$.
Lemma (\ref{lemma:as-rmse-diff}) shows that the difference between $f_\theta$ and $f_{\theta+\Delta}$ can be lower-bounded in terms of their sensitivities.
First, we note, as $n \rightarrow \infty$, with $\alpha=2$,
\begin{equation}
\frac{as_\seqlen(f)}{2\seqlen-\frac{as_\seqlen(f)}{\|f\|_2^2}} = \frac{\mu }{2-\frac{\mu}{\|f\|_2^2}}  + o(1) \geq \frac{\mu}{2} + o(1)
\end{equation}
where we used $\frac{\mu}{\|f\|_2^2} \in [0,1]$ up to $o(1)$.
Indeed,  if the output of $\pm 1$, we have $\|f\|_2^2=1$, and the term indeed is lower-bounded by $\mu$:
\begin{equation}\label{eq:improved-bound-boolean}
\frac{as_\seqlen(f)}{2\seqlen-\frac{as_\seqlen(f)}{\|f\|_2^2}} = \frac{\mu }{2-\mu}  + o(1) \geq \mu + o(1)
\end{equation}
showing that the $2$ factor can be eliminated in the case of Boolean output.
	Next,  as $n \rightarrow \infty$, for $\alpha=2$ as above, and using $0 \leq \frac{as_\seqlen(T_\theta)}{n} \leq \|T_\theta\|_2^2 \leq 1$, we find
\begin{equation}\label{eq:expected-deviation-bound}
\begin{aligned}
	& \mathbb{E}_\Delta\mathbb{E}[(T_\theta(x)-T_{\theta+\Delta}(x))^2] \\
	\geq  &\mathbb{E}_\Delta\underbrace{\frac{(1-\frac{1}{\alpha}) as_\seqlen(T_\theta)}{(\seqlen-\frac{as_\seqlen(T_\theta)}{\alpha \|T_\theta(x)\|_2^2})}}_{\text{independent of $\Delta$}} + \underbrace{\|T_\theta\|_2^2 \alpha \mathbb{E}_\Delta\frac{as_\seqlen(T_{\theta+\Delta})}{{as_\seqlen(T_\theta)}}}_{\geq 0} - 2  \sqrt{\frac{(1-\frac{1}{\alpha}) as_\seqlen(T_\theta)}{(\seqlen-\frac{as_\seqlen(T_\theta)}{\alpha \|T_\theta(x)\|_2^2})}}  \|T_\theta\|_2 \mathbb{E}_\Delta\sqrt{\alpha \frac{as_\seqlen(T_{\theta+\Delta})}{{as_\seqlen(T_\theta)}}} \\
	\geq  & \frac{ as_\seqlen(T_\theta)}{2\seqlen-\frac{as_\seqlen(T_\theta)}{\|T_\theta(x)\|_2^2}} - 2  \sqrt{\frac{1}{(\seqlen-\frac{1}{2 \|T_\theta(x)\|_2^2})}}  \|T_\theta\|_2 \mathbb{E}_\Delta\sqrt{ as_\seqlen(T_{\theta+\Delta})} \\
	\geq  & \frac{ as_\seqlen(T_\theta)}{2\seqlen-\frac{as_\seqlen(T_\theta)}{\|T_\theta(x)\|_2^2}} - 2  \mathbb{E}_\Delta\sqrt{\frac{as_\seqlen(T_{\theta+\Delta})}{(\seqlen-\frac{1}{2 \|T_\theta(x)\|_2^2})}}   \\
	\geq^{(\dagger)}  & \frac{ as_\seqlen(T_\theta)}{2\seqlen-\frac{as_\seqlen(T_\theta)}{\|T_\theta(x)\|_2^2}} - 4  \mathbb{E}_\Delta\sqrt{\frac{as_\seqlen(T_{\theta+\Delta})}{\seqlen}}   \\
	=  & \frac{ as_\seqlen(T_\theta)}{2\seqlen-\frac{as_\seqlen(T_\theta)}{\|T_\theta(x)\|_2^2}} - 4  \mathbb{E}_\Delta\sqrt{\frac{\mathbb{E}_x s_\seqlen(T_{\theta+\Delta},x)}{\seqlen}}   \\
	\geq  & \frac{\mu}{2} - 4  \sqrt{\frac{ \mathbb{E}_\Delta  \mathbb{E}_x  s(T_{\theta+\Delta},x)}{n}} + o(1) \\
\end{aligned}
\end{equation}
	where $(\dagger)$ uses $(\seqlen-\frac{1}{2 \|T_\theta(x)\|_2^2}) \geq \seqlen-1/2 \geq \seqlen/4$, and the final step used Jensen's inequality applied to the concave function $\sqrt{\cdot}$.
	By the reasoning above, in Eq.~(\ref{eq:improved-bound-boolean}), the factor 2 in the first term can be eliminated when the output is $\pm 1$.
	In order to lower-bound (\ref{eq:expected-deviation-bound}), we want to show the following:

\emph{
Take any $\zeta > \frac{2L}{d}$.
For each $x$, with probability at least $(1-\mathcal{O}(n^{2-(d-1)\zeta/L}) - \numLayers\exp(-\Theta(d)))$ (where $\mathcal{O}$ contains constants depending on $\rho, d, \numLayers$, but not depending on $x$ and $n$), the blowup is simultaneously bounded on $x$ and its radius-1 Hamming ball:}
\begin{equation}\label{eq:tau-perturbed-bound}
 \tau^{(k)}_{\theta+\Delta}(X) \leq 1+n^{\zeta/L}, \ \ \ \ \ \forall X \in \{ x, x\flipBit{1}, x\flipBit{2}, \dots, x\flipBit{\seqlen}\}
\end{equation}
We will prove this below in Step 4.
We may then bound $ \tau^{(k)}_{\theta+\Delta}(X) \leq 2n^{\zeta/L}$ for large $n$.
By (\ref{eq:main-formula-localized}), if (\ref{eq:tau-perturbed-bound}) is proven, then:
\begin{align*}
    \frac{s(T_{\theta+\Delta},x)}{C_{\theta+\Delta}\sqrt{\seqlen\log\seqlen}}  
   & \leq  \Blowup(x)^2 +\frac{1}{\seqlen} \sum_{i=1}^{\seqlen} \Blowup(x\flipBit{i})^2 
   \leq  4 n^{2\zeta} 
\end{align*}
and hence
\begin{equation}\label{eq:raw-sens-perturbed-over-n}
  s(T_{\theta+\Delta},x)  
    \leq  4 n^{2\zeta+.5} C_{\theta+\Delta}  \sqrt{ \log{\seqlen}}
\end{equation}
and hence
\begin{equation}\label{eq:sens-perturbed-over-n}
\begin{aligned}
  \frac{s(T_{\theta+\Delta},x)}{n}  
   & \leq  4 n^{2\zeta-.5} C_{\theta+\Delta} \sqrt{\log{\seqlen}} \\
\end{aligned}
\end{equation}
with probability $1-\frac{H}{\seqlen^{2}}$ over $x$ and $1-\mathcal{O}(\numLayers n^{2-(d-1)\zeta/\numLayers}) - \numLayers\exp(-\Theta(d))$ over $\Delta$, where we used a union bound over the $\numLayers$ layers.
Let us now fix $\zeta = \frac{3L}{d}$. Then, as $n\rightarrow \infty$,
\begin{align*}
    \mathbb{E}_\Delta\mathbb{E}_x  \frac{s(T_{\theta+\Delta},x)}{n} \leq & 4 n^{6L/d-.5}  C_{\theta+\Delta} \sqrt{ \log{\seqlen}} + \mathcal{O}(n^{2-3\frac{d-1}{d}}) + \frac{H}{\seqlen^{2}} + \numLayers \exp(-\Omega(d)) +  o(1)
\end{align*}
Under the hypothesis $12L < d$, all terms except for $\numLayers\exp(-\Omega(d))$ are $o(1)$ as $n\rightarrow \infty$.
Inserting this into (\ref{eq:expected-deviation-bound}), the claim follows.

\paragraph{Step 4: Effect on Layer Norm}
It remains to show (\ref{eq:tau-perturbed-bound}).
We proceed inductively from the input layer upwards, examining the effect of perturbing parameters.
In the inductive step, we consider a transformer where all the parameters below the bias immediately preceding the $i$-th layer norm have already been perturbed.

Fix an input string $x \in \{\pm 1\}^\seqlen$.
Write $\ypre_i^{(k)}$ for the activation computed using the perturbed parameters $\theta+\Delta$ for all parameters except the bias of interest, for which $\theta$ is used here. Thus, after applying the perturbation to that bias, the activation is $\ypre_i^{(k)}+\Delta_i$.
The aim is to lower-bound the SD of $\ypre_i^{(k)}+\Delta_i$ with high probability over the choice of $\Delta$.

Let $center(x) := x-mean(x)$, i.e. the pre-scaling component of layer norm.
First,
\begin{align*}
\mathbb{P}\left(\|center(\Delta_i)\| \leq \frac{\rho}{2} \sqrt{\frac{d}{D}}\right) & \leq \exp(-\Omega(d))
\end{align*}
by standard concentration arguments, where $\Omega(d)$ scales positively with $d$.\footnote{One way of obtaining this is by parameterizing $\Delta$ in terms of first sampling $D$ independent $X_1, \dots, X_D \sim \mathcal{N}(0,\frac{1}{D})$, and dividing the resulting vector by its norm, before multiplying by $\rho$ in the end. Before normalization, the squared norms of the entire vector and the norm of the part corresponding to $\Delta_i$ concentrate with exponential tail bounds in $d$, around $1$ and $\frac{d}{D}$, respectively. Furthermore, the mean of $\Delta_i$ concentrates around 0. Thus, the norm of $center(\Delta_i)$ concentrates around $\rho \sqrt{\frac{d}{D}}$.}
Let $\Pi_{\mathbb{S}} : \mathbb{R}^d \rightarrow \mathbb{S}^{d-1} \subset \mathbb{R}^d$ be the projection on the sphere around 0 of radius $\frac{\rho}{2} \sqrt{\frac{d}{D}}$:
\begin{equation}
    \Pi_{\mathbb{S}}(x) = \frac{\rho}{2} \sqrt{\frac{d}{D}} \frac{x}{\|x\|}
\end{equation}
Let $\Pi_{\mathbb{B}} : \mathbb{R}^d \rightarrow B_{\frac{\rho}{2} \sqrt{\frac{d}{D}}}(0)$ be the projection on the ball around 0 of radius $\frac{\rho}{2} \sqrt{\frac{d}{D}}$.
$\Pi_{\mathbb{B}}$ is a contraction (referred to below with ($\dagger$)\footnote{The contractivity of $\Pi_{\mathbb{B}}$ is proven, in a different context, as Lemma A.9 in \citet{edelman2022inductive}.}).
Conditional on $\|center(\Delta_i)\| \geq \frac{\rho}{2} \sqrt{\frac{d}{D}}$, we have $\Pi_{\mathbb{B}}center(\Delta_i) = \Pi_{\mathbb{S}}center(\Delta_i)$ and its distribution is, by symmetry, uniform on the ball's surface.
We now bound the probability that the SD of the perturbed activation is large:
\begin{align*}
& \mathbb{P}\left(\exists j : SD(\Delta_i + \ypre_i^{(k)}) < n^{-\zeta/L} \right) \\
 =   & \mathbb{P}\left(\exists j : \|center(\Delta_i + \ypre_i^{(k)})\| < \sqrt{d} n^{-\zeta/L} \right) \\
    = & \mathbb{P}\left(\exists j : center(\Delta_i) \in B_{n^{-\zeta/L}}\left(-center( \ypre_i^{(k)})\right) \right) \\
    \leq & \sum_{j=1}^\seqlen \mathbb{P}\left(center(\Delta_i) \in B_{\sqrt{d}n^{-\zeta/L}}\left(-center( \ypre_i^{(k)})\right) \right) \\
    \leq & \mathbb{P}(\|center(\Delta_i)\| \leq \frac{\rho}{2} \sqrt{\frac{d}{D}}) + \sum_{j=1}^\seqlen \mathbb{P}\left(center(\Delta_i) \in B_{\sqrt{d}n^{-\zeta/L}}\left(-center( \ypre_i^{(k)})\right) \bigg| \|center(\Delta_i)\| > \frac{\rho}{2} \sqrt{\frac{d}{D}} \right)\\
    \leq^{\dagger} & \mathbb{P}(\|center(\Delta_i)\| \leq \frac{\rho}{2} \sqrt{\frac{d}{D}}) + \sum_{j=1}^\seqlen \mathbb{P}\left(\Pi_{\mathbb{B}} center(\Delta_i) \in B_{\sqrt{d}n^{-\zeta/L}}\left(-\Pi_{\mathbb{B}} center( \ypre_i^{(k)})\right) \bigg| \|center(\Delta_i)\| > \frac{\rho}{2} \sqrt{\frac{d}{D}} \right)\\
    \leq & \exp(-\Omega(d)) + \sum_{j=1}^\seqlen \mathbb{P}\left(\Pi_{\mathbb{B}} center(\Delta_i) \in B_{\sqrt{d}n^{-\zeta/L}}\left(-\Pi_{\mathbb{B}} center( \ypre_i^{(k)})\right) \bigg| \|center(\Delta_i)\| > \frac{\rho}{2} \sqrt{\frac{d}{D}} \right)\\
    \leq & \exp(-\Omega(d)) + \sum_{j=1}^\seqlen \mathbb{P}\left(\Pi_{\mathbb{S}} center(\Delta_i) \in B_{\sqrt{d}n^{-\zeta/L}}\left(-\Pi_{\mathbb{S}} center( \ypre_i^{(k)})\right) \bigg| \|center(\Delta_i)\| > \frac{\rho}{2} \sqrt{\frac{d}{D}} \right)\\
    = & \exp(-\Omega(d)) + \sum_{j=1}^\seqlen \frac{Area\left(B_{\sqrt{d}n^{-\zeta/L}}\left(-\Pi_{\mathbb{S}} center( \ypre_i^{(k)})\right)\right)}{Area\left(\frac{\rho}{2} \sqrt{\frac{d}{D}} \mathbb{S}\right)}\\
\end{align*}
where $Area(B_{\sqrt{d}n^{-\zeta/L}}\left(-\Pi_{\mathbb{S}} center( \ypre_i^{(k)})\right))$ is the area of a hypersphere cap of radius $\sqrt{d}n^{-\zeta/L}$ on a hypersphere of radius $\frac{\rho}{2} \sqrt{\frac{d}{D}}$, which is at most $n^{-(d-1)\zeta/L}$ times some constant depending on $d, D, \rho$ but not $n, x$\footnote{A nonasymptotic formula for the area of the hypersphere cap is given by \citet{li2010concise}. A simple derivation of the asymptotic relationship used here is by approximating, when $n$ is large and thus $n^{-\zeta/L}$ is small in relation to $d$, the cap as a disk on the $d-1$-dimensional tangent space to the hypersphere. This approximation is valid in this limit since the radius of the hypersphere, and thus its curvature, are independent of $n$.}. The denominator is constant in $n$. Overall, as $n\rightarrow\infty$,
\begin{align*}
    \mathbb{P}\left(\exists j : \|center(\Delta_i + \ypre_i^{(k)})\| < n^{-\zeta/L} \right) \leq \exp(-\Omega(d)) + \mathcal{O}(n^{1-(d-1)\zeta/L})
\end{align*}
If the event denoted by this probability is avoided, then under the parameter setting $\theta+\Delta$, 
\begin{equation}\label{eq:tau-layer-norm-bound}
    \tau_k(x) \leq 1+ n^{\zeta/L}
\end{equation}
for both $x$ and its Hamming ball of radius 1, 
with probability (by a union bound over the radius-1 Hamming ball $B_1(x)$ around $x$):
\begin{align*}
   & \mathbb{P}\left(\forall x' \in B_1(x) : \forall  j : \|center(\Delta_i + \ypre_i^{(k)}(x'))\| \geq n^{-\zeta/L}  \right) \\
   \geq & 1-\mathbb{P}\left(\|center(\Delta_i)\| \leq \frac{\rho}{2} \sqrt{\frac{d}{D}}\right)-\sum_{x' \in B_1(x)} \mathbb{P}\left(center(\Delta_i) \in B_{n^{-\zeta/L}}\left(-center( \ypre_i^{(k)}(x'))\right) \bigg| \|center(\Delta_i)\| > \frac{\rho}{2} \sqrt{\frac{d}{D}} \right) \\
   \geq & (1-\mathcal{O}(n^{2-(d-1)\zeta/L}) - \exp(-\Theta(d))) 
\end{align*}
This concludes the proof.

\paragraph{Remark}
Note that we assumed for notational simplicity that layer norm only applies a single time, following the MLP (\ref{eq:layer-norm-and-mlp}).
Depending on the implementation, LN may appear in other cases, e.g. directly following attention (e.g., in \citet{vaswani2017attention}).
In this case, we can  analogously consider the the effect of $V + \Delta_V$, using the properties of the random matrix $\Delta_V$ and the effects of the perturbation on the preceding MLP (or the token embeddings).
First, we note that $\eta_i := \sum_h \sum_j \widehat{a}^{(k,h)}_{ij} \ypost_{j}^{(k-1)}$ is, after adding $\Delta$ to the parameters involved, very unlikely to have norm $o(n^{1/(\zeta L)})$, because of the random perturbation applied to the MLP bias in the preceding layer (or to the token embeddings, if $k=1$), provided $d \gg H$.
Second, considering the spectral properties of the random matrix $\Delta_V$,  the perturbed head activation $\ypost_{i}^{(k-1)} + (V+\Delta_i)\eta_i$ is unlikely to have much smaller SD, hence a bound analogous to (\ref{eq:tau-layer-norm-bound}) follows with high probability.

\end{proof}

\section{Theorem~\ref{thm:scratchpad}}\label{sec:proof:scratchpad}

\begin{proof}
    Each autoregressive step consists in determining, based on a sequence $x_1\dots x_N {t}_1 \dots {t}_t$ the state ${t}_{t+1} = T({t}_t, x_t)$. Assuming $\log |\Sigma|$ and $\log |X|$ bits are used to encode $x_i$ and ${t}_i$, respectively, the function thus only depends on $\log |\Sigma||X|$ inputs, and is represented by a polynomial of that degree, independent of $N$.
\end{proof}

\section{Further Experimental Results}

\subsection{The Details of Experimental Setup}
\label{app:hyperparameters}

\begin{table}[h]
    \centering
    \begin{minipage}{0.45\textwidth}
        \centering
        \begin{tabular}{ll}
            \hline
            \textbf{Parameter Name} & \textbf{Value} \\ \hline
            Learning Rate  & 0.0003  \\
            Weight Decay & 0.1 \\
            Batch Size     & 1024    \\
            Training Steps         & 10,000    \\
            Optimizer      & AdamW   \\ 
            $\rho$ (in computation of sharpness) & 0.02 \\
            \hline
        \end{tabular}
    \end{minipage}
    \hfill
    \begin{minipage}{0.45\textwidth}
        \centering
        \begin{tabular}{ll}
            \hline
            \textbf{Parameter Name} & \textbf{Value} \\ 
            \hline
            Hidden Dimension & 128 \\
            Transformer Layers & 2 \\
            Attention Heads & 2 \\
            Dropout Rate & 0 \\
            \hline
        \end{tabular}
    \end{minipage}
    \caption{In all the experiments, the hyperparameters above were used unless stated otherwise.}
\end{table}

In the experiment presenting the tradeoff between weight norm and LayerNorm blowup, we uniformly sampled weight decay between 0 and 0.4 and learning rate between 0.0001 and 0.0005.

In the scratchpad experiment, we used an encoder-decoder Transformer with 2 attention heads per block, hidden dimension 32 and 2 layers.

In the experiment including training the models for high sequence lengths (Figures \ref{app:exp-scaling-higher-lengths} and \ref{app:exp-scaling-higher-lengths-aligned}) we used hidden dimension 4.

\subsection{Results}
\label{appendix:exp_results}

\begin{figure}[h!]
    \centering
    \begin{subfigure}[b]{0.47\linewidth}
        \includegraphics[width=\linewidth]{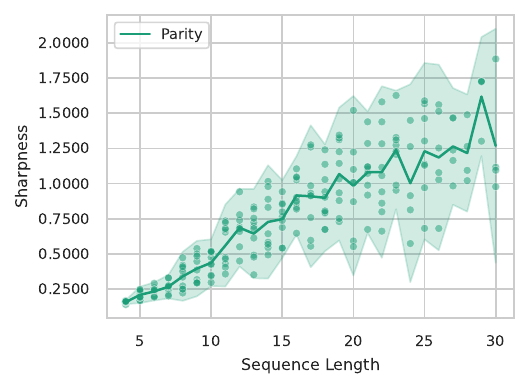}
        \caption{PARITY}
    \end{subfigure}
    \hfill
    \begin{subfigure}[b]{0.47\linewidth}
        \includegraphics[width=\linewidth]{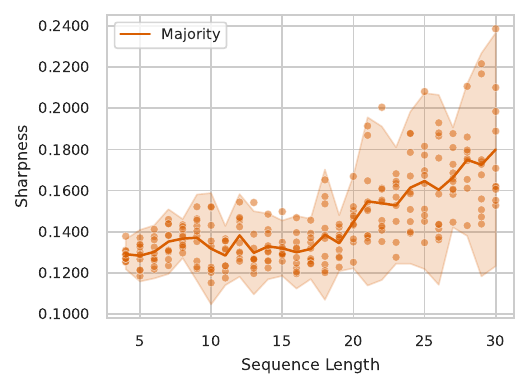}
        \caption{MAJORITY}
    \end{subfigure}
    
    \begin{subfigure}[b]{0.47\linewidth}
        \includegraphics[width=\linewidth]{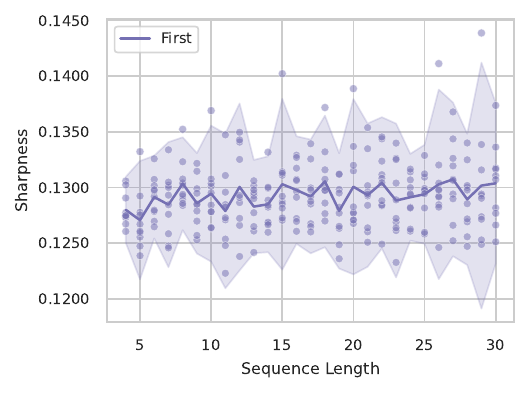}
        \caption{FIRST}
    \end{subfigure}
    \hfill
    \begin{subfigure}[b]{0.47\linewidth}
        \includegraphics[width=\linewidth]{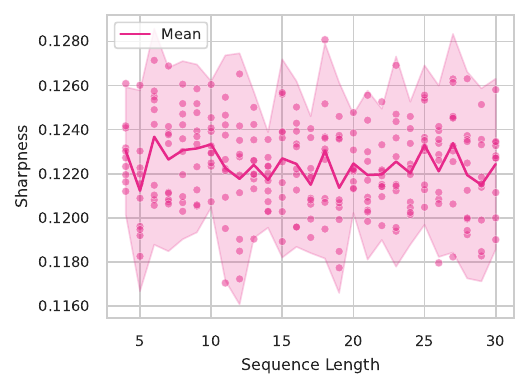}
        \caption{MEAN}
    \end{subfigure}
    
    \caption{Sharpness as a function of sequence length for all the functions discussed in the paper. As predicted by Theorem~\ref{thm:lrho-bound}, parameters fitting PARITY have substantial sharpness as inputs get longer. For functions with lower sensitivity, sharpness barely increases with the input length.
    For PARITY, the sharpness approaches the theoretical asymptotic lower bound of $1$  from Theorem~\ref{thm:lrho-bound} already at $\seqlen \approx 30$.
    See Figure~\ref{ex:sharpness-all-aligned} for a version with aligned y-axes.
    }
    \label{exp:sharpness-all}
\end{figure}

\begin{figure}
    \centering
    \begin{subfigure}[b]{0.47\linewidth}
        \includegraphics[width=\linewidth]{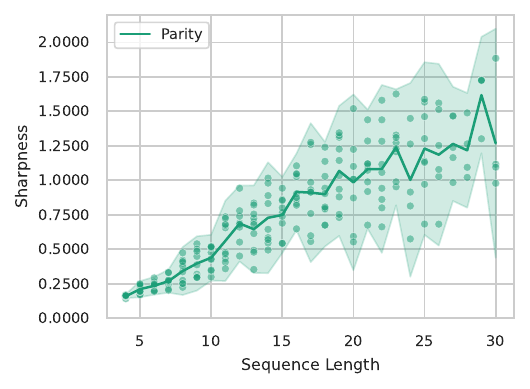}
        \caption{PARITY}
    \end{subfigure}
    \hfill
    \begin{subfigure}[b]{0.47\linewidth}
        \includegraphics[width=\linewidth]{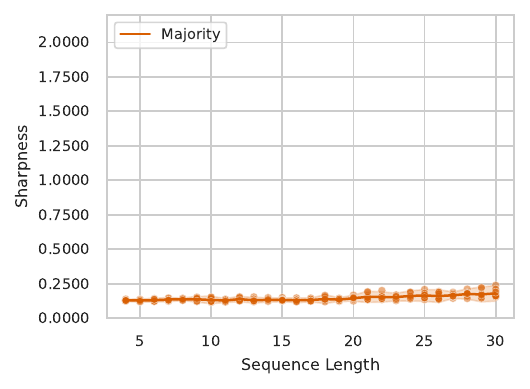}
        \caption{MAJORITY}
    \end{subfigure}
    
    \begin{subfigure}[b]{0.47\linewidth}
        \includegraphics[width=\linewidth]{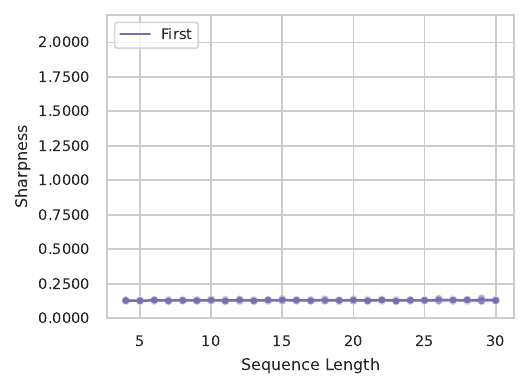}
        \caption{FIRST}
    \end{subfigure}
    \hfill
    \begin{subfigure}[b]{0.47\linewidth}
        \includegraphics[width=\linewidth]{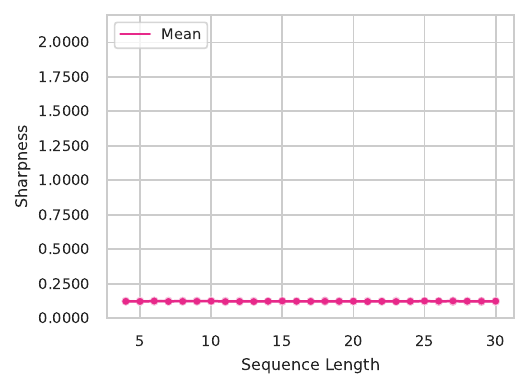}
        \caption{MEAN}
    \end{subfigure}
    
    \caption{Same as Figure \ref{exp:sharpness-all}, but $y$-axes are aligned.}\label{ex:sharpness-all-aligned}
\end{figure}

\begin{table}
\centering
\begin{tabular}{lccc}
\toprule
\textbf{Function} & \makecell{\textbf{Regression} \\ \textbf{Slope}} & \makecell{\textbf{Pearson} \\ \textbf{Correlation}} & \makecell{\textbf{p-value} \\ \textbf{($H_0: \mathrm{slope} = 0$)}}\\
\midrule \addlinespace
PARITY & $5.0 \cdot 10^{-2}$ & 0.88 & $6 \cdot 10^{-72}$ \\
MAJORITY & $1.8 \cdot 10^{-3}$ & 0.67 & $2 \cdot 10^{-36}$ \\
FIRST & $6.4 \cdot 10^{-5}$ & 0.16 & $0.0075$ \\
MEAN & $-1.9 \cdot 10^{-5}$ & -0.07 & 0.22 \\
\bottomrule
\end{tabular}
\caption{The statistical properties of the length-sharpness relationship computed for $n \in [4, 30]$. For all the functions except for MEAN, sequence length and minima sharpness are statistically significantly correlated within this range of $\seqlen$; though for PARITY the slope is 2 orders of magnitude higher than for other functions. See the visualisations in Figure \ref{exp:sharpness-all}.}
\label{exp:sharpness-stats-table}
\end{table}

\begin{figure}
    \centering
    \begin{subfigure}[b]{0.47\linewidth}
        \includegraphics[width=\linewidth]{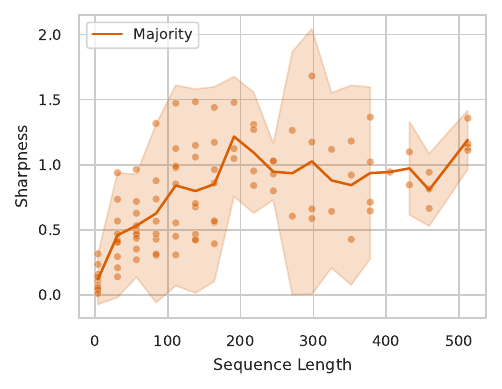}
        \caption{MAJORITY}
    \end{subfigure}
    \hfill
    \begin{subfigure}[b]{0.47\linewidth}
        \includegraphics[width=\linewidth]{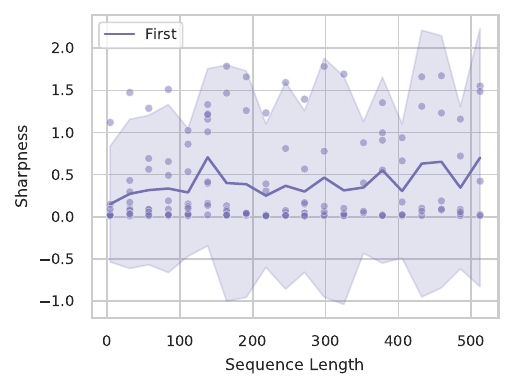}
        \caption{FIRST}
    \end{subfigure}

    \begin{subfigure}[b]{0.47\linewidth}
        \includegraphics[width=\linewidth]{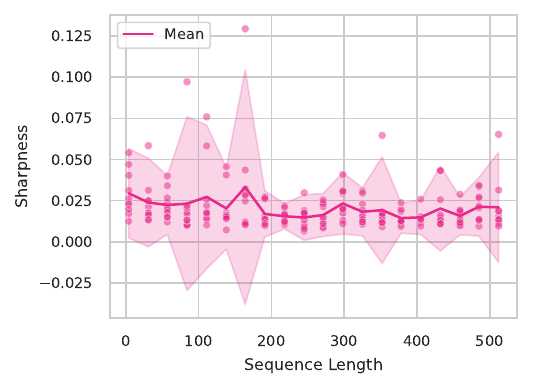}
        \caption{MEAN}
    \end{subfigure}
    \caption{Length-Sharpness dependency for high sequence lengths. Note that training a comparable setup for PARITY at such input lengths was not feasible. 
    For MAJORITY, sharpness does increase, in line with its higher (though still sublinear) asymptotic average sensitivity.
    For FIRST and MEAN, whose average sensitivity does not increase with $\seqlen$, sharpness shows little discernible dependency on length.
    We note that the absolute sharpness values for high input lengths (as in this figure) and for low input lengths (as in Figure \ref{exp:sharpness-all}) are incomparable, as different model hyperparameters were used, due to the computational cost of fitting models at high sequence lengths.
    }
    \label{app:exp-scaling-higher-lengths}
\end{figure}

\begin{figure}
    \centering
    \begin{subfigure}[b]{0.47\linewidth}
        \includegraphics[width=\linewidth]{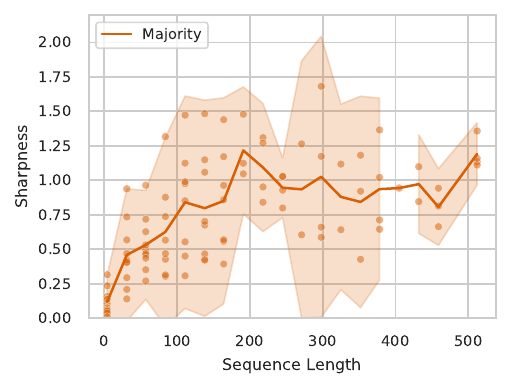}
        \caption{MAJORITY}
    \end{subfigure}
    \hfill
    \begin{subfigure}[b]{0.47\linewidth}
        \includegraphics[width=\linewidth]{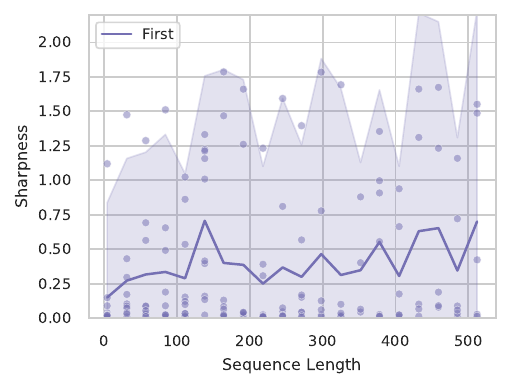}
        \caption{FIRST}
    \end{subfigure}

    \begin{subfigure}[b]{0.47\linewidth}
        \includegraphics[width=\linewidth]{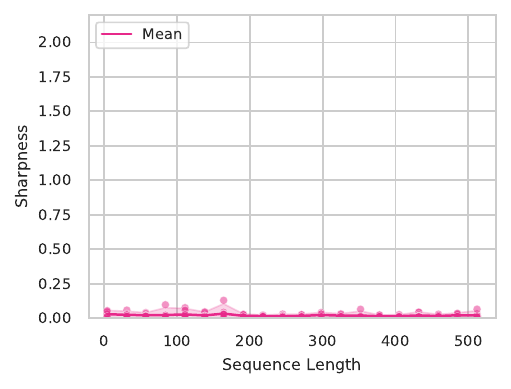}
        \caption{MEAN}
    \end{subfigure}
    \caption{Same as Figure \ref{app:exp-scaling-higher-lengths}, but $y$-axes are aligned.}
    \label{app:exp-scaling-higher-lengths-aligned}
\end{figure}

\begin{figure}
    \centering
    \begin{subfigure}[b]{0.47\linewidth}
        \includegraphics[width=\linewidth]{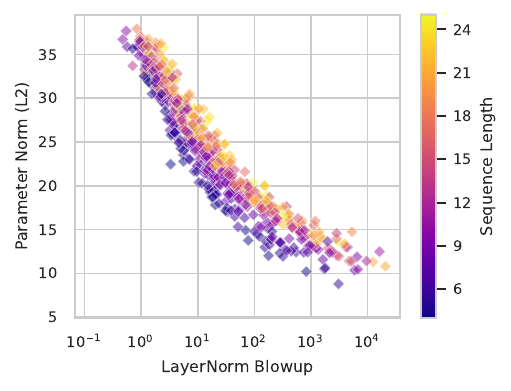}
        \caption{PARITY}
    \end{subfigure}
    \hfill
    \begin{subfigure}[b]{0.47\linewidth}
        \includegraphics[width=\linewidth]{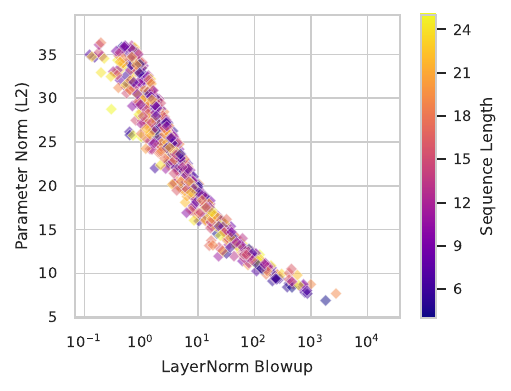}
        \caption{MAJORITY}
    \end{subfigure}
    
    \begin{subfigure}[b]{0.47\linewidth}
        \includegraphics[width=\linewidth]{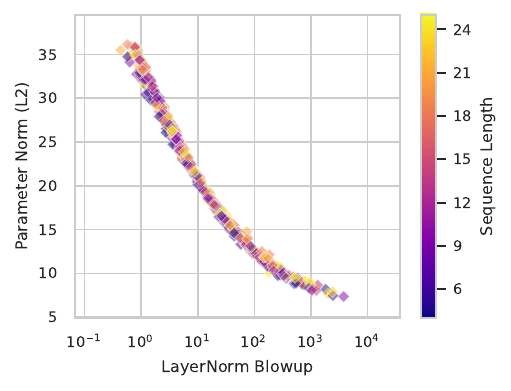}
        \caption{FIRST}
    \end{subfigure}
    \hfill
    \begin{subfigure}[b]{0.47\linewidth}
        \includegraphics[width=\linewidth]{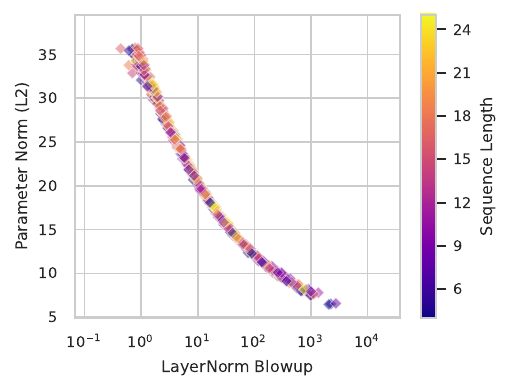}
        \caption{MEAN}
    \end{subfigure}

    \caption{Tradeoff between layer norm blowup and parameter norm: When varying the layer norm penalty, transformers find different tradeoffs between these two quantities. For PARITY, the tradeoff depends on the input length; blowup or parameter weights need to increase with the input length (in accordance with Corollary \ref{thm:bigtheorem}).
    For functions with lower sensitivity, little dependency on the input length is observed.
    }
    \label{exp:tradeoff-all-functions-comparison}
\end{figure}

\begin{figure}
    \centering
    \includegraphics[scale=0.9]{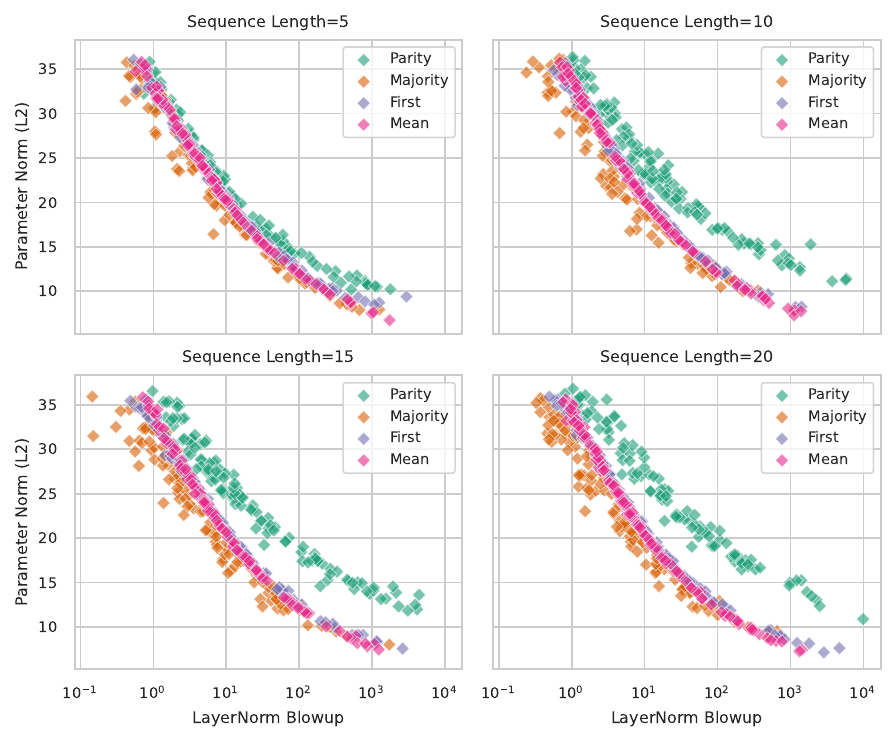}
    \caption{Parameter Norm/Blowup tradeoff for input lengths 5, 10, 15 and 20. At input length 5, PARITY behaves similarly to the other functions. However, for higher input lengths, fixed LN Blowup requires higher weight norm for PARITY than for other functions, and vice versa. This aligns with the theory, as increasing input length for PARITY increases the lower bound of weight norm and LN blowup in (\ref{eq:theorem-third}).}
    \label{exp:tradeoff-4-lengths}
\end{figure}

\begin{figure}
    \centering
    \includegraphics[width=0.4\textwidth]{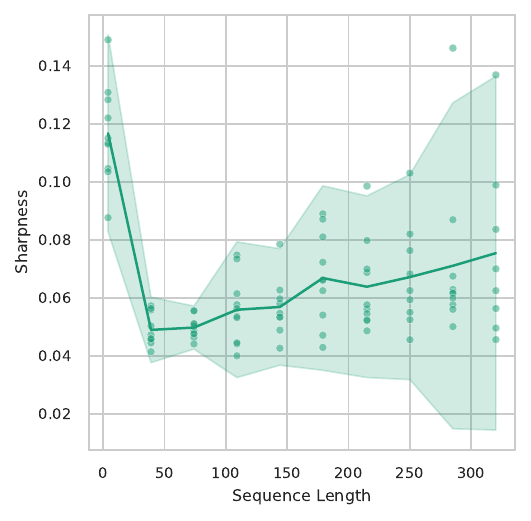}
    \caption{Sharpness of a Transformer with scratchpad trained for PARITY. Despite approximating a highly sensitive function, sharpness stays low even at hundreds of bits and shows no clear increase with length. 
    }
    \label{exp:scratchpad-sharpness}
\end{figure}

\begin{figure}
    \centering
    \begin{subfigure}[b]{0.47\linewidth}
        \includegraphics[width=\linewidth]{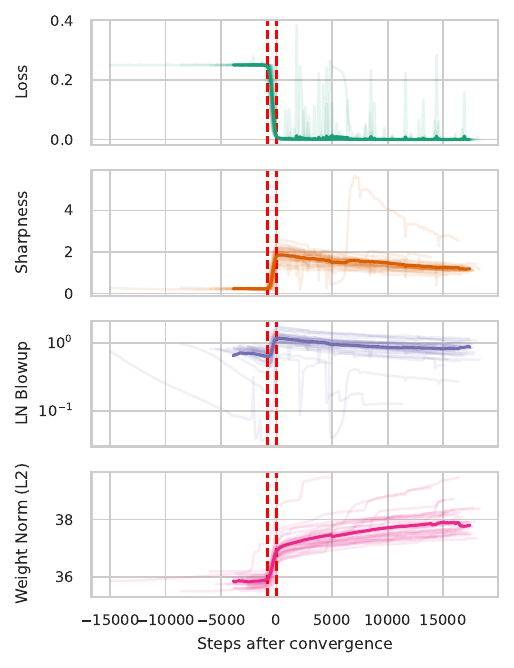}
        \caption{Weight Decay = 0.0}
    \end{subfigure}
    \hfill
    \begin{subfigure}[b]{0.47\linewidth}
        \includegraphics[width=\linewidth]{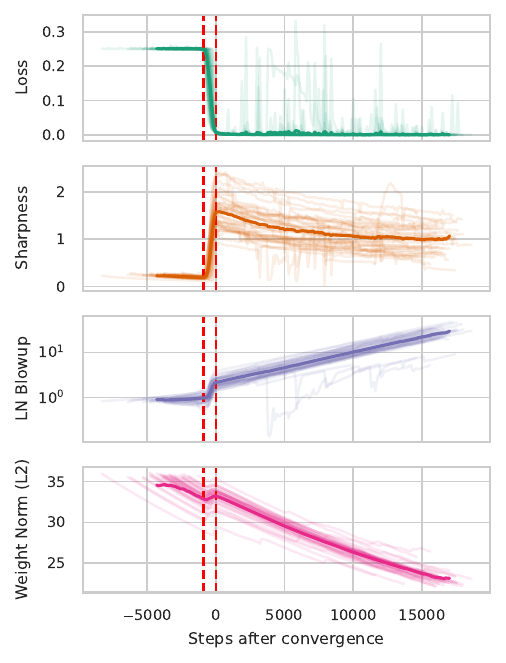}
        \caption{Weight Decay = 0.1}
    \end{subfigure}
    
    \begin{subfigure}[b]{0.47\linewidth}
        \includegraphics[width=\linewidth]{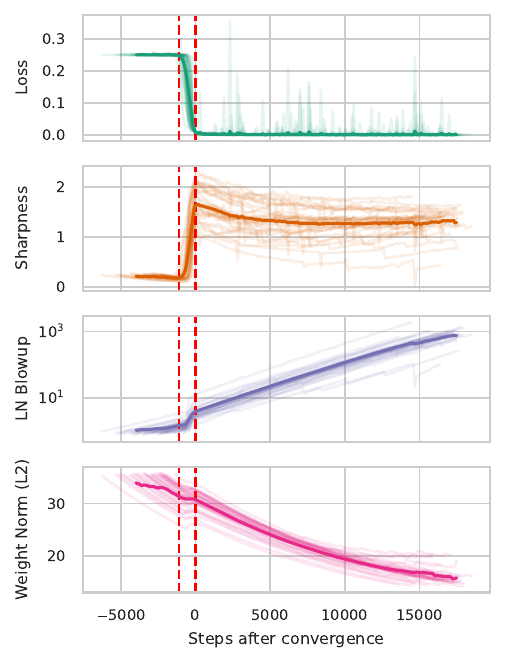}
        \caption{Weight Decay = 0.2}
    \end{subfigure}
    \hfill
    \begin{subfigure}[b]{0.47\linewidth}
        \includegraphics[width=\linewidth]{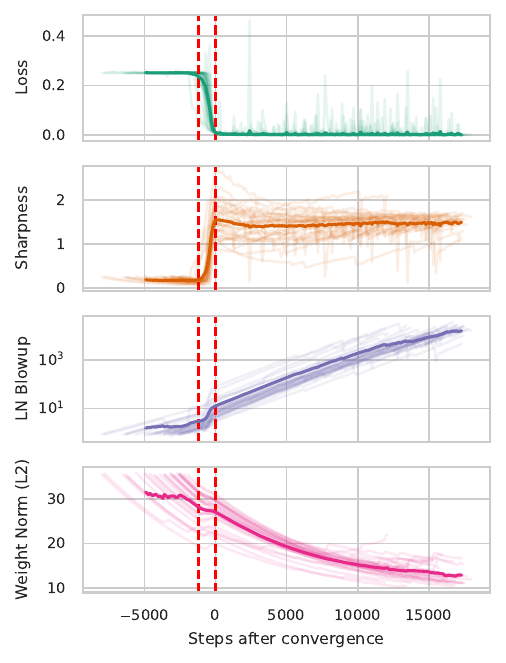}
        \caption{Weight Decay = 0.3}
    \end{subfigure}

    \caption{Training dynamics of a Transformer for PARITY at different values of weight decay. Sharpness always increases dramatically at the same time as training loss falls to 0. For non-zero weight decay, after the learning phase parameter norm starts to decrease, causing the appropriate increase in LayerNorm blowup, which does not have a significant effect on sharpness.}

    \label{exp:dynamic-main}
\end{figure}

\begin{figure}
    \centering
    \begin{subfigure}[b]{0.47\linewidth}
        \includegraphics[width=\linewidth]{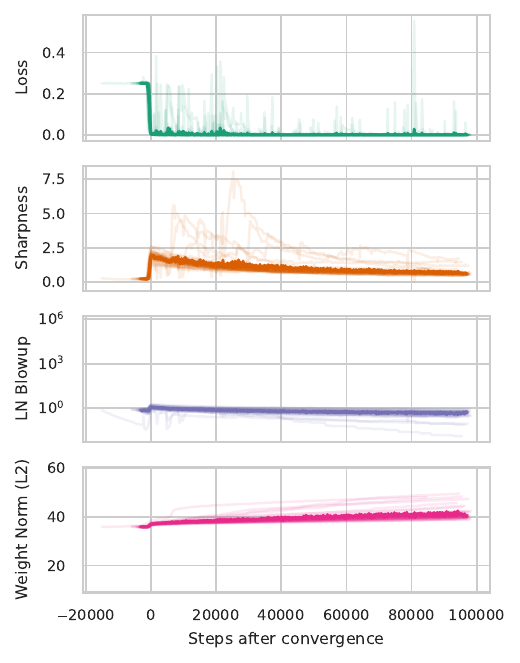}
        \caption{Weight Decay = 0.0}
    \end{subfigure}
    \hfill
    \begin{subfigure}[b]{0.47\linewidth}
        \includegraphics[width=\linewidth]{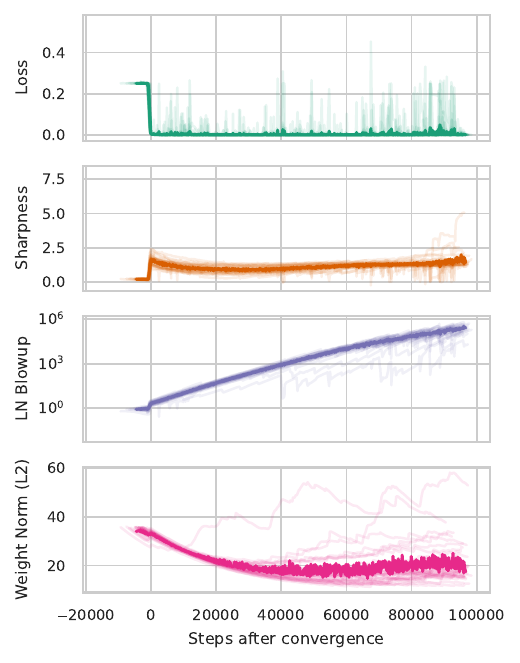}
        \caption{Weight Decay = 0.1}
    \end{subfigure}

    \caption{Training dynamics at weight decay 0.0 (left) and 0.1 (right), 100k training steps, PARITY. As predicted by Corollary~\ref{thm:bigtheorem}, over the course of training, layer norm blowup and parameter norm trade off. Zero weight decay ultimately enables a lower sharpness, but even after 100k steps, it remains substantially higher than the low sharpness quickly reached on low-sensitivity functions. }\label{fig:dynamics-100k}
\end{figure}

\begin{figure}
    \centering
    \begin{subfigure}[b]{0.47\linewidth}
        \includegraphics[width=\linewidth]{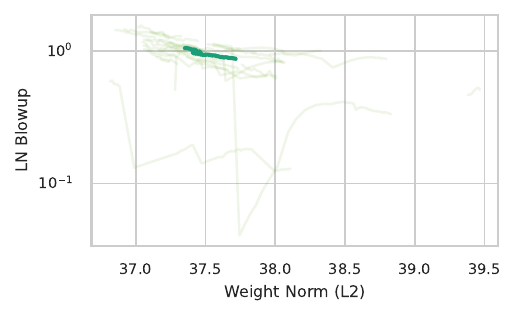}
        \caption{Weight Decay = 0.0}
    \end{subfigure}
    \hfill
    \begin{subfigure}[b]{0.47\linewidth}
        \includegraphics[width=\linewidth]{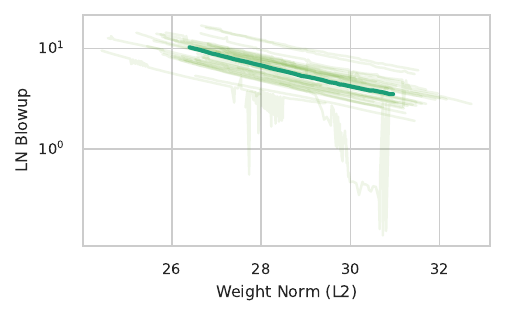}
        \caption{Weight Decay = 0.1}
    \end{subfigure}
    
    \begin{subfigure}[b]{0.47\linewidth}
        \includegraphics[width=\linewidth]{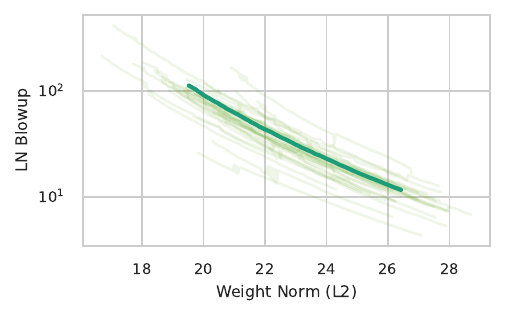}
        \caption{Weight Decay = 0.2}
    \end{subfigure}
    \hfill
    \begin{subfigure}[b]{0.47\linewidth}
        \includegraphics[width=\linewidth]{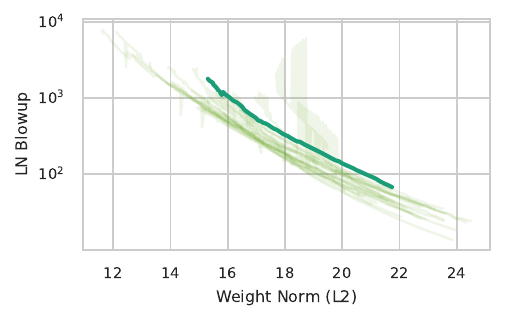}
        \caption{Weight Decay = 0.3}
    \end{subfigure}

    \caption{Tradeoff between Weight Norm and LayerNorm Blowup when training with different weight decay values. In the course of training, non-zero weight decay drives down parameter norm, and that makes LN Blowup to increase. The observed dependency is log-linear, similar to the results in Figure \ref{exp:tradeoff-4-lengths}.}
\end{figure}

\begin{figure}
\centering
\includegraphics{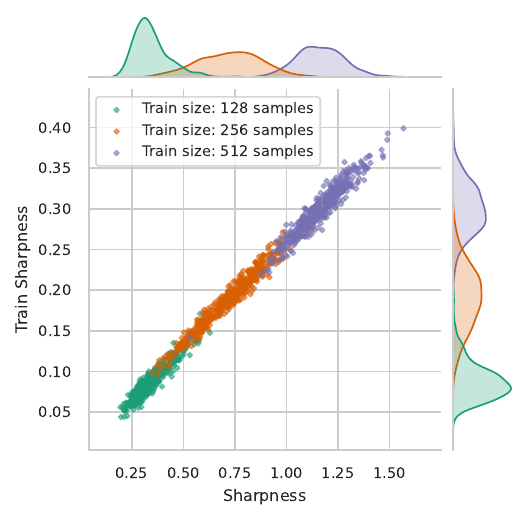}
    \caption{Generalization: The sharpness of the model when restricting sharpness (\ref{eq:def:lrho}) to the train set is almost the same as the sharpness on the whole input space. }\label{fig:generalization-sharpness}
\end{figure}

\end{document}